%% file: main_arxiv.tex
\documentclass[11pt,oneside,letter]{article}
\usepackage[top=1 in, bottom=1 in, left=0.85 in, right=0.85 in]{geometry}
\usepackage{hyperref}
\usepackage{amsmath}
\usepackage[round]{natbib}
\usepackage[table]{xcolor}
\usepackage{xcolor}
\usepackage{amsthm} 
\usepackage{microtype}
\usepackage{tabularx}
\usepackage{subcaption}
\usepackage{algorithm}
\usepackage[noend]{algpseudocode}
\usepackage{amssymb} 
\usepackage{graphicx}
\usepackage{multirow}
\usepackage{pifont}
\usepackage{booktabs} 
\usepackage{enumitem}

\input{custom}
\allowdisplaybreaks

\newtheorem{innercustomgeneric}{\customgenericname}
\providecommand{\customgenericname}{}
\newcommand{\newcustomtheorem}[2]{%
  \newenvironment{#1}[1]
  {%
   \renewcommand\customgenericname{#2}%
   \renewcommand\theinnercustomgeneric{##1}%
   \innercustomgeneric
  }
  {\endinnercustomgeneric}
}
\theoremstyle{plain}
\newtheorem{theorem}{Theorem}[section]

\newtheorem{lemma}[theorem]{Lemma}

\theoremstyle{definition}

\theoremstyle{remark}

\newcustomtheorem{customthm}{Theorem}
\newcustomtheorem{customlemma}{Lemma}
\newtheorem{cor}[theorem]{Corollary}

\title{Model-Free, Regret-Optimal Best Policy Identification in Online CMDPs}
\author{
    Zihan Zhou\footnote{The work was completed when Zihan was a visiting student at the University of Michigan, Ann Arbor.}\\
    IIIS, Tsinghua University\\
    \texttt{zh-zhou20@mails.tsinghua.edu.cn}
    \and
    Honghao Wei\\
    Washington State University\\
    \texttt{honghao.wei@wsu.edu}
    \and
    Lei Ying \\
    University of Michigan, Ann Arbor \\  
    \texttt{leiying@umich.edu}\\
}

\begin{document}

\maketitle

\begin{abstract} 
This paper considers the best policy identification (BPI) problem in online Constrained Markov Decision Processes (CMDPs). We are interested in algorithms that are model-free, have low regret, and identify an approximately optimal policy with a high probability. 
Existing model-free algorithms for online CMDPs with sublinear regret and constraint violation do not provide any convergence guarantee to an optimal policy and provide only average performance guarantees when a policy is uniformly sampled at random from all previously used policies. In this paper, we develop a new algorithm, named Pruning-Refinement-Identification (PRI), based on a fundamental structural property of CMDPs proved in \cite{Koo_88,Ros_89}, which we call {\em limited stochasticity}.  The property says for a CMDP with $N$ constraints, there exists an optimal policy with {\em at most} $N$ stochastic decisions. 
The proposed algorithm first identifies at which step and in which state a stochastic decision has to be taken and then fine-tunes the distributions of these stochastic decisions. Assuming the CMDP instance is well-separated\footnote{The exact definition can be found in Section \ref{sparsity}.}, PRI achieves trio objectives: (i) PRI is a model-free algorithm; and (ii) it outputs an approximately optimal policy with a high probability at the end of learning; and (iii) PRI guarantees $\tilde{\mathcal{O}}(H\sqrt{K})$ regret and zero constraint violation for well separated CMDPs\footnote{{\bf Notation:} $f(n) = \tilde{\mathcal O}(g(n))$ denotes $f(n) = {\mathcal O}(g(n){\log}^k n)$ with $k>0.$ The same applies to $\tilde{\Omega}.$}, which significantly improves the best existing regret bound $\tilde{\mathcal{O}}(H^4 \sqrt{SA}K^{\frac{4}{5}})$  under a model-free algorithm, where $H$ is the length of each episode, $S$ is the number of states, $A$ is the number of actions, and the total number of episodes during learning is $2K+\tilde{\cal O}(K^{0.25}).$ We further present a matching lower via an example that shows under any online learning algorithm, there exists a well-separated CMDP instance such that either the regret or violation has to be $\Omega(H\sqrt{K}),$ which matches the upper bound by a polylogarithmic factor.
\end{abstract}

\section{Introduction}
In unconstrained reinforcement learning (RL), an agent aims at an optimal policy that maximizes the accumulated reward by interacting with a stochastic environment. RL has achieved remarkable successes in multiple areas, including industrial process optimization, robotics, and gaming. \citep{rajawat2023cognitive, abeyruwan2023sim2real, lindegaard2023intrinsic, liu2023exploring}. However, in many real-world applications, the learned policy must also satisfy a set of constraints. For example, in healthcare applications, we need to optimize patient treatment plans while considering constraints like medication dosage, scheduling of medical procedures, and resource allocation in hospitals. 

These constrained versions of RL problems can be formulated as Constrained Markov Decision Processes (CMDPs) \citep{Alt_99}. Learning in CMDPs has become an active research topic recently. Existing solutions include both model-based algorithms (see e.g. \citep{BraDudLyk_20,EfrManPir_20,SinGupShr_20,LiuZhoKal_21,BurHasKal_21,DinWeiYan_20,CheJaiLuo_22}) and model-free algorithms (see e.g., \citep{GhoZhoShr_22,WeiLiuYin_22-2,WeiLiuYin_22,WeiGhoShr_23}). This paper focuses on model-free approaches for CMDPs due to their computation and memory efficiency. A fundamental drawback of existing model-free solutions for online CMDPs is that they provide only average performance guarantees for a policy uniformly sampled at random from {\em all} previously used policies during learning, so they fail to identify a single optimal or a near-optimal policy.\footnote{In this paper, a policy is a mapping from a state at a given step to an action distribution, without any other additional input information. An algorithm that uses multiple policies, e.g. randomly sampling one policy from many policies, is called a {\em mixed} policy in this paper.} Therefore, a natural question arises:
\begin{center}
    {\bf Is it possible to identify an optimal or an approximately optimal policy in online CMDPs with a model-free approach with optimal regret and constraint violation?}
\end{center}

There are two key challenges to answering this question: $(i)$ The optimal solution to a CMDP problem is a stochastic policy in general. Model-free online CMDP algorithms often employ the primal-dual approach, utilizing Lagrange multipliers to balance reward maximization and constraint violation. However, these methods yield greedy policies for fixed Lagrange multipliers. A greedy policy is not optimal in general. Consequently, model-free algorithms such as Triple-Q \cite{WeiLiuYin_22} only offer performance guarantees when averaging over a large number of greedy policies determined by different Lagrange multipliers, failing to converge to a single policy. $(ii)$ The best-known regret bound of model-free algorithms for episodic, online CMDPs is $\tilde{\mathcal{O}}(K^{\frac{4}{5}})$ \cite{WeiLiuYin_22}. However, it is known that model-based algorithms can achieve a smaller and order-wise tight regret $\tilde{\mathcal{O}}(\sqrt{K})$ \cite{EfrManPir_20}. The open question is whether a model-free algorithm can achieve $\tilde{\mathcal{O}}(\sqrt{K})$ regret in online CMDPs? 

This paper tackles both challenges, providing affirmative responses to both questions. We introduce a novel algorithm, Pruning-Refinement-Identification (PRI), inspired by a fundamental CMDP property \cite{Koo_88,Ros_89}, which states that for an episodic CMDP with $N$ constraints, there exists an optimal policy that makes stochastic decisions in at most $N$ step-dependent states out of the $HS$ step-dependent states.

Based on this insight, PRI consists of three phases. In this first phase (pruning), PRI identifies when and where stochastic decisions are necessary. This defines a set of greedy policies that together approximate a ``mixed'' optimal policy. The subsequent refinement phase involves learning the weights of these greedy policies. This is done through iterative optimization, utilizing empirical reward and utility value functions. The process refines value function estimates with each iteration, aiming to minimize regret.  In the final identification phase, PRI learns the occupancy measure, determining the probability of visiting specific state-action pairs at each step. This information is used to recover a single policy from the near-optimal mixed policy obtained during the refinement phase. The main contributions of this paper are summarized below.  
\begin{itemize}[leftmargin=*]
    \item PRI is the first model-free, Probably Approximately Correct (PAC) RL algorithm for online CMDPs. For well separated CMDPs, PRI achieves $\tilde{\mathcal{O}}(H\sqrt{K})$ regret and zero constraint violation, which significantly improves the best existing regret bound $\tilde{\mathcal{O}}(H^4\sqrt{SA}K^{\frac{4}{5}})$  under a  mode-free algorithm. Unlike existing regret bounds, the dominating term (in terms of $K$) in the regret bound does not depend on the state and action space sizes. We also provide a matching lower bound $\Omega(H\sqrt{K})$ for well separated CMDPs.
    \item PRI outputs a near-optimal policy with a high probability at the end. The learned policy guarantees $\tilde{\mathcal{O}}(1/\sqrt{K})$ optimality gap for the reward value function and zero constraint violation with probability $1 - \tilde{\mathcal{O}}(1/\sqrt{K})$.
\end{itemize}

\section{Related Work}
{\bf Best policy identification in MDPs.}
For unconstrained MDPs, existing studies on BPI focus on $(\epsilon, \delta)$-PAC RL algorithms, i.e., algorithms that identify an $\epsilon$-optimal policy with probability at least $1-\delta.$ Such a learning objective has been considered extensively in discounted and episodic tabular MDPs \citep{AgaKakKri_20,AzaRemHil_13,EveEyaShi_06,DomMenKau_21,HeZhoGu_21,SidWanWu_18}. A recent work \cite{TauJedPro_22} also studied BPI in linear MDPs, which has a sample complexity of $O\left(\frac{1}{\epsilon^2}\right)$.
This paper considers BPI for online CMDPs using a model-free approach. To the best of our knowledge, it remains an open problem. 

{\bf Model-based and model-free algorithms for online CMDPs.}
As mentioned in the introduction, most existing results on online CMDPs consider regret minimization instead of BPI. For example, \cite{BraDudLyk_20,EfrManPir_20,SinGupShr_20} proposed model-based algorithms for episodic tabular CMDPs. \cite{LiuZhoKal_21,BurHasKal_21} proposed efficient algorithms with zero or bounded constraint violation. For model-free algorithms, \cite{WeiLiuYin_22-2} developed Triple-Q that achieves sublinear regret and zero constraint violation in episodic tabular CMDPs. Similar results have been established for linear CMDPs \citep{GhoZhoShr_22,DinWeiYan_20}  and infinite-horizon average CMDPs \citep{CheJaiLuo_22,WeiLiuYin_22}. However, these existing model-free algorithms for online CMDPs does not converge to an optimal or a near-optimal policy. Note that model-free algorithms have a memory complexity of ${\cal O}(HSA)$ for maintaining the Q-table while the memory complexity of model-based algorithms is ${\cal O}(HS^2A)$ for maintaining the transition kernel. Very recently, \cite{MosOdoVee_23} considered BPI for online CMDPs. They formulated the CMDP problem as a min-max game and the proposed algorithm converges to a near-optimal policy at the last iteration with optimistic mirror descent. However, the paper does not provide any regret guarantee when learning the near-optimal policy. There are also algorithms for the average-reward CMDP problem, including model-based approaches \cite{AgaBaiAgg_21,AgaBaiAgg_22,ZheRat_20} and model-free approaches \cite{CheJaiLuo_22,WeiLiuYin_22-2}. These algorithms do not identify the optimal policy at the end. Table \ref{ta:algorithms} summarizes the recent results for online, episodic CMDPs.
\begin{table*}[!t]
\centering
 \caption{The Exploration-Exploitation Tradeoff in Episodic CMDPs.}
 \label{ta:algorithms}
 \begin{center}
 \resizebox{\textwidth}{25mm}{
  \begin{tabular}{|c|l|l|l|l|}
   \toprule
   & {\bf Algorithm} & {\bf Regret} & {\bf Constraint Violation} & {\bf BPI?}\\
   \hline
    \multirow{7}* {Model-based} & OPDOP \citep{DinWeiYan_20} & $\tilde{\mathcal{O}}(H^3\sqrt{S^2AK})$& $\tilde{\mathcal{O}}(H^3\sqrt{S^2AK})$ & \ding{55} \\\cline{2-5}
    &  OptDual-CMDP \citep{EfrManPir_20}  & $\tilde{\mathcal{O}}(H^2\sqrt{S^3AK})$ & $\tilde{\mathcal{O}}(H^2\sqrt{S^3AK})$ & \ding{55} \\\cline{2-5}
    &  OptPrimalDual-CMDP \citep{EfrManPir_20}  & $\tilde{\mathcal{O}}(H^2\sqrt{S^3AK})$  & $\tilde{\mathcal{O}}(H^2\sqrt{S^3AK})$ & \ding{55} \\
   \cline{2-5}
    & {CONRL \citep{BraDudLyk_20}}  & $\tilde{\mathcal{O}}(H^3\sqrt{S^3A^2K})$& $\tilde{\mathcal{O}}(H^3\sqrt{S^3A^2K})$ & \ding{55}\\
    \cline{2-5} 
    & {OptPess-LP \citep{LiuZhoKal_21} } & $\tilde{\mathcal{O}}(H^3\sqrt{S^3AK})$  &  $0 $ & \ding{55}\\
    \cline{2-5} 
    &{OptPess--PrimalDual \citep{LiuZhoKal_21} }& $\tilde{\mathcal{O}}(H^3\sqrt{S^3AK})$  &  $\mathcal{O}(1)$ & \ding{55} \\
    \cline{2-5} 
    &{OPSRL\citep{BurHasKal_21}} & $\tilde{\mathcal{O}}(\sqrt{S^4H^7AK})$  &  $0$ & \ding{55} \\
  \bottomrule
   \toprule
  \multirow{2} * {Model-free} & Triple-Q\citep{WeiLiuYin_22} &$\tilde{\cal O}(\frac{1 }{\delta}H^4 S^{\frac{1}{2}}A^{\frac{1}{2}}K^{\frac{4}{5}} )$ &$0$ & \ding{55} \\
  \cline{2-5}
  & \textbf{PRI} (this paper) &$\tilde{\cal{O}} (H\sqrt{K })$ &$0$ & \ding{51}\\
   \bottomrule
  \end{tabular}}
 \end{center}
\end{table*}

\section{Problem Formulation}
We consider an episodic CMDP, denoted by $(\mathcal{S,A},H, \mathbb{P}, r, g^{(n)}, n\in[N])$, where $\mathcal{S}$ is the state space ($\vert \mathcal{S}\vert = S$), $\mathcal{A}$ is the action space ($\vert \mathcal{A} \vert = A$), $\{r_h\}_{h=1}^H, \{g^{(n)}_h\}_{h=1}^H,n\in[N] $ are reward, $n$-th utility functions, and $\mathbb{P} = \{\mathbb{P}_h(\cdot\vert x,a) \}_{h=1}^H$ are the transition kernels. For simplicity, we assume that in each episode, the agent starts from the same initial state $x_1=x_{ini}$. It is straightforward to generalize the results to the case when the initial state is sampled from a given distribution but the notation becomes cumbersome. We also assume that $r_h: \cS \times \cA \rightarrow [0,1]$ and $g^{(n)}_h: \cS \times \cA \rightarrow [0,1]$ are deterministic for notation simplicity. Our results can be easily generalized to random reward/utility signals.

During each episode, the agent interacts with the environment as follows: at each step $h$, the agent takes action $a_h$ after observing state $x_h$, receives reward $r_h(x_h, a_h)$ and $N$ utility values $g^{(n)}_h(x_h, a_h)$ ($n \in [N]$ is the index of the utility functions), and then observes a new state ($x_{h+1}$), which evolves by following the transition kernel $\mathbb{P}_h(\cdot \vert x_h, a_h)$. The episode terminates after $H$ steps. 

Given a stochastic policy $\pi$, which is a collection of $H$ functions $\{\pi_h: \mathcal{S \times A} \rightarrow [0, 1]\}_{h=1}^H$, the agent takes action $a$ with probability $\pi_h(a \vert x)$ when being in state $x$ at  step $h$ . The reward value function of policy $\pi,$ denoted by $V^\pi_h(x),$ is the expected total reward when starting from an arbitrary state $x$ at step $h$ to the end of the episode: 
$$V^\pi_h(x) = \mathbb{E}_{\pi} \left[ \left.\sum_{i = h}^H r_i(x_i, a_i) \right\vert x_h = x\right],$$
where the expectation is taken with respect to the policy $\pi$ and randomness from the transition kernels. Accordingly, the reward Q-function, denoted by $Q^\pi_h(x,a),$ is the expected total reward when the agent starts from an arbitrary action-action pair $(x,a)$ at step $h$ and follows policy $\pi$ to the end of the episode:  
$$Q^\pi_h(x,a) = r_h(x, a)+\mathbb{E}_{\pi}\left[ \left.\sum_{i = h+1}^H r_i(x_i, a_i) \right\vert \begin{aligned} x_h=&x, \\ a_h=&a\end{aligned} \right].$$

Similarly, we can define the $n$th utility value functions as $$W^{\pi,n}_h(x) = \mathbb{E}_\pi \left[ \left.\sum_{i = h}^H g^{(n)}_i(x_i, a_i) \right\vert x_h = x\right] $$ and utility Q-functions as
$$C^{\pi,n}_h(x,a) = g^{(n)}_h(x, a) + \mathbb{E}_\pi \left[ \left.\sum_{i = h+1}^H g^{(n)}_i(x_i, a_i) \right\vert  \begin{aligned} x_h=&x,\\a_h=&a \end{aligned}  \right].$$
Given the definitions above, we have
\begin{align}
    V^\pi_h(x) &= \sum_a \pi_h(a|x) Q^\pi_h(x,a) \\
    Q^\pi_h(x,a) &= r_h(x,a) + \sum_{x'} \mathbb{P}_h(x'|x,a) V^\pi_{h+1} (x')\\
    W^{\pi,n}_h(x) &= \sum_a \pi_h(a|x) C^{\pi,n}_h(x,a) \\  C^{\pi,n}_h(x,a) &= g^{(n)}_h(x,a) + \sum_{x'} \mathbb{P}_h(x'|x,a) W^{\pi,n}_{h+1} (x').
\end{align}
The objective of the CMDP is to find an optimal policy that maximizes the expected total reward while making sure the $n-$th expected total utility is no less than $\rho^{(n)}$ for all $n\in[N]$:
\begin{equation}
\begin{aligned}
    &\pi^* \in  \arg\max_\pi V_1^\pi(x_{ini}) \\
    \text{s.t.} \quad &W_1^{\pi,n}(x_{ini}) \geq \rho^{(n)} \quad \forall n \in [N].
\end{aligned} \label{eq:cmdp}
\end{equation}
To avoid triviality, we assume $\rho^{(n)} \in [0,H]$. For simplicity, we use $V_1^\pi$ to represent $V_1^\pi(x_{ini})$ and $W_1^{\pi,n}$ to represent $W_1^{\pi,n}(x_{ini})$.

We evaluate an online RL algorithm for CMDP using regret and constraint violation over $K$ episodes:
\begin{align}
   & \text{Regret}(K) =  K V_1^{\pi^*} - \mathbb{E}\left[\sum_{k=1}^K V_1^{\pi_k}\right]\\
    &\text{Violation}^n(K) = K \rho^{(n)} - \mathbb{E}\left[\sum_{k=1}^K W_1^{\pi_k,n} \right], \label{eq:violation}
\end{align} where $\pi_k$ is the policy used in episode $k.$

\section{PRI (Pruning-Refinement-Identification)}\label{sparsity}
Before formally introducing our algorithm, we first present two structural properties of the optimal solution to the CMDP problem (\ref{eq:cmdp}), which serve as the foundation of our proposed algorithm. Consider a CMDP problem with $N$ constraints. It is well-known that the problem can be formulated as a linear programming (LP) problem \cite{Alt_99}:
    \begin{align}
   & \max_{\{q_h(x,a)\}} \sum_{h,x,a} q_h(x,a)r_h(x,a) \label{eq:cmdp-offline} \\
\hbox{s.t.: } &\sum_{h,x,a} q_{h}(x,a)g^{(n)}_h(x,a) \geq \rho^{(n)} \quad \forall n\in[N]\label{eq:constraint}  \\
    &\forall x\in{\cal S}, h\in [H]\nonumber \\
    &\sum_a q_{h+1}(x,a) = \sum_{x', a'} \mathbb{P}_{h} (x|x',a') q_{h}(x',a')\label{eq:transition} \\
    &\sum_a q_1(x_{ini},a) = 1, \sum_a q_1(x,a) = 0, \ x\not=x_{ini} \label{eq:init-distr} \\
    &q_h(x,a) \geq 0, \label{eq:proba} 
\end{align}
where $q_h(x,a)$ denotes the probability that state-action pair $(x,a)$ is visited at step $h$, called the occupancy measure. Each feasible solution $\{q_h(x,a)\}_{h,x,a}$ to the problem leads to a corresponding Markov policy: $\pi_h(a|x)=\frac{q_h(x,a)}{\sum_a q_h(x,a)}.$ In this paper, we call probability distribution $\pi_h(\cdot|x)$ {\em decision} at state $x$ at step $h.$ So a policy consists of $S\times H$ decisions. A decision $\pi_h(\cdot|x)$ is called {\em greedy} if $\pi_h(a|x)=1$ for some $a\in {\cal A}$ and stochastic otherwise. 

\begin{lemma}[Limited Stochasticity]\label{le:spar}
If $q^*=\{q^*_{h}(x,a)\}_{h,x,a}$ is an optimal solution to the CMDP problem (\ref{eq:cmdp-offline})-(\ref{eq:proba}) and is an extreme point, then there are at most $HS+ N$ nonzero values in $q^*$. This implies that the optimal policy derived from $q^*$ includes at most $N$ stochastic decisions. 
\end{lemma}
This result was proved in \cite{Koo_88,Ros_89}. A proof has been included in Appendix \ref{sub5} for the completeness of the paper. The following corollary, which is a well-known result, is a direct consequence of the lemma. 
\begin{cor}\label{cor:mdp}
    For unconstrained MDP problems, one of the optimal policies is a greedy policy.
\end{cor}
Given an occupancy measure $q$ and its induced policy $\pi,$ we define 
${\cal D}_{h,x}(q)=\left\{a: q_h(x,a)>0\right\},$ which is the set of actions that will be taken with a nonzero probability in state $x$ at step $h$ under the policy $\pi$ induced by $q.$ Note that if $\pi_h(\cdot|x)$ is a greedy decision, then $|{\cal D}_{h,x}(q)|=1;$ and if $\pi(\cdot|x)$ is stochastic, then $\vert {\cal D}_{h,x}(q)\vert >1.$  Let $M=\prod_{h,x}|{\cal D}_{h,x}(q)|,$ and let $\pi^m$ represent the $m$th greedy policy ($m=1, \cdots, M$) constructed from $\otimes_{h,x}{\cal D}_{h,x}(q)$  such that 
$\pi^m_h(a|x)=1$ only if $a\in {\cal D}_{h,x}(q).$ A greedy policy is a policy under which all decisions are greedy. Next, we will show that a Markov policy is equivalent to a mixed policy of many greedy policies in the following lemma, whose proof can be found in Appendix \ref{sub6}.

\begin{lemma}[Decomposition]
Given any Markov policy $\pi$ and its corresponding occupancy measure $q,$  there exists a set of $M$ greedy policies and a probability distribution $\{\alpha_m\}_{m=1, \cdots, M}$ such that the mixed policy, which selects a greedy policy $\pi^m$ at the start of an episode with probability $\alpha_m$ and subsequently follows it, has the same occupancy measure $q$ as the original policy $\pi$. 
\label{lem:decom}
\end{lemma}

Online model-free algorithms for CMDPs, such as Triple-Q \cite{WeiLiuYin_21}, guarantee sublinear regret and zero constraint violation on average but have no convergence guarantee. In fact, Triple-Q continues to adjust the dual variable (virtual queue) based on constraint violation, and when the dual variable is fixed (within a frame), the algorithm reduces to the traditional Q-learning. As suggested in the paper \cite{WeiLiuYin_21}, we can only recover a near-optimal policy by remembering all previous policies and then uniformly sampling one from them for each episode, i.e., a mixed policy of $K$ policies. Therefore, this near-optimal policy is a mixture of many, many greedy policies. More importantly, it is near-optimal only when averaging over a large number of episodes and may be far from optimal in each episode. 

Hence, unlike unconstrained MDPs where Q-learning converges to the optimal policy, finding a model-free algorithm that converges to the optimal policy or a near-optimal policy in CMDPs is highly nontrivial and remains to be an open problem. Lemma \ref{le:spar} (limited stochasticity), however, suggests that when the number of constraints, $N,$ is relatively small, solving an unknown CMDP may not differ significantly from solving an unknown MDP. This is because the majority of the decisions, specifically $HS-N$ out of the $HS$ decisions, are greedy and can be learned using traditional algorithms like Q-learning if we can first identify where the stochastic decisions need to be taken. Lemma \ref{lem:decom} further suggests that an optimal policy can be decomposed into $M$ greedy policies if all decision types are correctly identified, so we may recover an optimal or a near-optimal policy by evaluating the $M$ greedy policies.  

Leveraging these two observations from Lemma \ref{le:spar} and \ref{lem:decom}, we propose a novel three-phase algorithm (Algorithm \ref{alg1}), including policy pruning, policy refinement, and policy identification, called PRI.  
Policy Pruning includes no more than $8HSAK^{0.25}\log K$ episodes, while Policy Refinement and Policy Identification run $K$ episodes separately. The algorithm is used to solve a CMDP with tightened constraints $\tilde{\rho}^{(n}=\rho^{(n)} + \epsilon_\rho,$ where $\epsilon_\rho=\frac{\log^2 K}{\sqrt{K}},$ which allows us to reduce the constraint violation to zero while maintaining the regret at the same order. 

\begin{algorithm}[!htb]
    \caption{PRI}
    \begin{algorithmic}
        \State Consider CMDP with tightened constraints $\forall i \in [n], \tilde{\rho}^{(i)} =\rho^{(i)} + \epsilon_\rho,$ where $\epsilon_\rho=\tilde{\cal O}(1/\sqrt{K}).$
\State Run Policy Pruning (Alg. \ref{alg:mul}) to obtain $\{\tilde{\cal D}_{h,x}\}_{h,x}$
\State Run Policy Refinement (Alg. \ref{alg:ref}) with $\{\tilde{\cal D}_{h,x}\}_{h,x}$ to obtain $\{\alpha_i\}_{i=1}^m$
    \State Run Policy Identification (Alg. \ref{alg:ide}) with $\{\alpha_i\}_{i=1}^m$ to obtain policy $\tilde{\pi}.$
    \State Output policy $\tilde{\pi}.$
    \end{algorithmic}
    \label{alg1}
\end{algorithm}

\begin{algorithm}[!htb]
    \caption{Policy Pruning}
    \begin{algorithmic}
        \State Initialize $\tilde{\cal D}_{h,x} = {\cal A}$ for all $(h,x),$ set $z_h(x,a)=0$ and $\hbox{flag}(h,x,a)=0$ for all $h,$ $x$ and $a.$
       \For {$t=1, 2, ..., 4\log K$}
            \State Run Triple-Q for $K^{0.25}$ episodes  and record $N_h(x,a),$ the number of times action $a$ is used when the system is in state $x$ at step $h.$
            \For {all $h,$ $x,$ and $a$}
                \If{$N_h(x,a)\leq K^{0.2}$}
                    \State $z_h(x,a)\leftarrow z_h(x,a)+1$
                \EndIf
            \EndFor
        \EndFor
         \For {all $h,$ $x,$ and $a$}
            \If{$z_h(x,a)\geq 2\log K$}
            \State $\tilde{\cal D}_{h,x}\leftarrow \tilde{\cal D}_{h,x}\setminus \{a\}$  and $\hbox{flag}(h,x,a)=1.$
            \EndIf
        \EndFor  
            \While{$\exists \hbox{flag}(h', x', a') = 0$}
            \State Set $z=0,$ $\hbox{flag}(h', x', a') = 1,$ and $\hat{\cal D}=\otimes_{(h,x)\not=(h', x')}\tilde{\cal D}_{h,x}\otimes(\tilde{\cal D}_{h',x'}\setminus\{a'\})$
            \For {$t=1, 2, ..., 4\log K$}
                \State Update $z\leftarrow z+$Compare($\tilde{\cal D}$, $\hat{\cal D}$)(Alg. \ref{alg:comp})
            \EndFor
            \If {$z \geq 2{\log K}$}
                \State Update $\tilde{\mathcal{D}}_{h',x'} \leftarrow \tilde{\mathcal{D}}_{h',x'} \setminus \{a'\}$
            \Else 
                \State Set $z=0,$ $\hbox{flag}(h', x', a') = 1,$ and $\hat{\cal D}=\otimes_{(h,x)\not=(h', x')}\tilde{\cal D}_{h,x}\otimes\{a'\}$
                \For {$t=1, 2, ..., 4\log K$}
                    \State Update $z\leftarrow z+$Compare($\tilde{\cal D}$, $\hat{\cal D}$)(Alg. \ref{alg:comp})
                \EndFor
                \If {$z \geq 2{\log K}$}
                    \State Set $\hbox{flag}(h', x', a) = 1$ for all $a\in \tilde{\cal D}_{h', x'}$ and  update $\tilde{\mathcal{D}}_{h',x'} \leftarrow  \{a'\}$
                \EndIf
            \EndIf
            \EndWhile
        \State Return $\{\tilde{\cal D}_{h,x}\}_{h,x}$
    \end{algorithmic}
    \label{alg:mul}
\end{algorithm}

\begin{algorithm}[!htb]
    \caption{Compare($\tilde{\cal D},$ $\hat{\cal D}$)}
    \begin{algorithmic}
        \State Set $z = 0$ 
        \For {$t=1, 2, ..., 4\log K$}
            \State Run Triple-Q for $K^{0.25}$ episodes with action space $\tilde{\cal D}$. Record the average cumulative reward ${v}^*$.
            
            \State Reset Triple-Q and run it for $K^{0.25}$ episodes with action space $\hat{\cal D}$. Record the average cumulative reward $\tilde{v}$ and average cumulative utilities $\tilde{w}^n.$
            \If {$|v^* - \tilde{v}|\leq \frac{4}{K^{0.03}}$ and $\tilde{w}^n\geq \tilde{\rho}^{(n)}$ for all $n$}
                \State $z \leftarrow z + 1$
            \EndIf
        \EndFor
        \If {$z \geq 2{\log K}$}
            \State Return 1
        \Else
            \State Return 0
        \EndIf
    \end{algorithmic}
    \label{alg:comp}
\end{algorithm}

\begin{algorithm}[!htb]
    \caption{Policy Refinement}
    \begin{algorithmic}
    \State {\bf Input:}  $\{\tilde{\cal D}_{h,x}\}_{h,x}.$
        \State Obtain $M$ greedy policies from $\{\tilde{\cal D}_{h,x}\}_{h,x}$ where $M=\prod_{h,x} |\tilde{\cal D}_{h,x}|.$ 
        \State Set $\epsilon'=\frac{1}{\log K}$
    \If {$M=1$}
        \State Output the greedy policy directly.
    \Else
    \State Set $\hat{V}^{\pi^m}_1=0, \hat{W}^{\pi^m, n}_1=0,$  and $\alpha_m=\frac{1}{M}$ for all $n$ and $m.$    
    \For{ round $t=1, \cdots, \sqrt{K}$}
    \For{$m=1, \cdots, M$}
    \For{$k=1, \cdots, \alpha_m \sqrt{K}$}
    \State Execute greedy policy $\pi^m$ for one episode. 
    \If {$k\leq \epsilon' \sqrt{K}$} 
    \State Set  $\hat V^{\pi^m}_1\leftarrow \hat V^{\pi^m}_1+V^{\pi^m}_{k,1}$ and $\hat W^{\pi^m,n}_1\leftarrow \hat W^{\pi^m,n}_1+W^{\pi^m,n}_{k,1}$ for all $n,$ where  $ V^{\pi^m}_{k,1}$ and $W^{\pi^m,n}_{k,1}$ are the total reward and utility of type $n$ received in the $k$th episode.
    \EndIf
    \EndFor
    \State Set $\bar{V}^{\pi^m}_1 = \frac{\hat V^{\pi^m}_1}{t\epsilon' \sqrt{K}}$ and  $\bar{W}^{\pi^m, n}_1 = \frac{\hat W^{\pi^m, n}_1}{t\epsilon' \sqrt{K}}$  for all $n.$
    \EndFor
    \State Update $\{\alpha_m\}$ by solving Decomposition-Opt \eqref{decomp-opt}. 
    \EndFor
    \EndIf
    \State Return $\{\alpha_i\}_{i=1}^m$.
    \end{algorithmic}
    \label{alg:ref}
\end{algorithm}

\begin{algorithm}[!htb]
\caption{Policy Identification}
\begin{algorithmic}
\State {\bf Input:} $\{\alpha_i\}_{i=1}^m.$
    \State  Initialize ${N}_h(x,a)=0$ for all $h,$ $x$ and $a.$
    \For{$t=1, \cdots, \sqrt{K}$}
    \For{$m=1, \cdots, M$}
    \For{$k=1, \cdots, \alpha_m \sqrt{K}$}
    \For{$h=1, \cdots, H$}
        \State Take action $a_h$ given by policy $\pi^m,$ i.e. $\pi^m(a_h|x_h)=1.$ 
        \State ${N}_h(x_h,a_h)\leftarrow {N}_h(x_h,a_h) +1.$ 
    \EndFor
    \EndFor
    \EndFor
    \EndFor
    \State For all $(h,x,a)$, set $\tilde{\pi}_h(a|x)=\frac{N_h(x,a)}{\sum_{\tilde{a}\in{\cal A}}N_h(\tilde{a}, x)}.$
    \State Return $\tilde{\pi}$
\end{algorithmic}
\label{alg:ide}
\end{algorithm}

Denote $\tilde{\cal D}_{h,x}$ as the set of actions for state $x$ and step $h.$ 
The key idea of Policy Pruning (Algorithm \ref{alg:mul})  is to evaluate each remaining action in of $\tilde{\cal D}_{h',x'}$. The algorithm first decides whether  $a'\in \tilde{\cal D}_{h',x'}$ can be removed  while retaining at least one optimal policy for the tightened CMDP, i.e., with the following action space:
$$\otimes_{(h,x)\not=(h', x')}\tilde{\cal D}_{h,x}\otimes(\tilde{\cal D}_{h',x'}\setminus\{a'\}).$$
This is done by running Triple-Q with the above action space for $K^{0.25}$ episodes and comparing the average total reward with that obtained by running Triple-Q with the original action space for $K^{0.25}$ episodes.\footnote{A brief review can be found in the Appendix \ref{sec:ap-qqq}. We remark that PRI can be viewed as a ``meta-algorithm'' that builds on any model-free CMDP algorithm with sublinear regret and constraint violation.} If the difference is small and the tightened constraints are not violated, then with probability $1 - {\cal O}(K^{-0.02})$, at least one of the optimal policies is retained so we can remove action $a'$ from $\tilde{\cal D}_{h',x'}.$ To improve the probability, the algorithm repeats the comparison for $4\log K$ times and uses the majority rule to decide whether to remove $a'$. In that case, with probability $1 - {\cal O}(K^{-\frac{9}{8}})$, the decision is made correctly. 

If $a'$ is not removed, then any optimal policy in $\otimes_{(h,x)}\tilde{\cal D}_{h,x}$ has to use action $a'$ in state $x'$ at step $h'$. Policy Pruning next determines whether using $a'$ alone is sufficient, i.e., whether an optimal policy is retained in the following action space
$$\otimes_{(h,x)\not=(h', x')}\tilde{\cal D}_{h,x}\otimes(\tilde{\cal D}_{h',x'}=\{a'\}).$$
This is again done by running Triple-Q with the above action space for $K^{0.25}$ episodes and comparing the average total reward with that under Triple-Q with the original action space. If the difference is small and the tightened constraints are not violated, then with a high probability, one of the optimal policies takes a greedy decision at $(h',x')$ with action $a'.$ The algorithm repeats $4\log K$ comparisons and uses the majority rule to make the final decision on $a'.$ If using $a'$ alone is not sufficient, the algorithm keeps $a'$ in $\tilde{\cal D}_{h',x'}$ and moves to a different action in $\tilde{\cal D}_{h',x'}.$ Note that we use Triple-Q for ${\cal O}(K^{0.25}\log K)$ episodes for each  action instead of $K$ episodes because it is easier to learn whether an optimal policy still exists than learning the actual optimal policy.

After the first phase, PRI obtains $M$ greedy policies from the remaining actions. In the second phase, PRI learns the weights $\{\alpha_m\}$ so that a mixed policy that chooses policy $\pi^m$ with probability $\alpha_m$ is statistically identical to the optimal policy for the tightened CMDP. This is achieved by learning the reward and utility value functions of the greedy policies and then solving an approximated version of the CMDP (Decomposition-Opt \eqref{decomp-opt}).

At each round of the second phase (policy refinement), the following optimization with $M$ optimization variables is solved. 
\begin{equation}
    \begin{aligned}
        &\hbox{\bf Decomposition-Opt:}\\
        & \max_{\{a_m\}_{m=1}^M} \sum_{m=1}^{M} \alpha_m\bar{V}^{\pi^m}_1\\
        \hbox{s.t.:} & \left|\sum_{m=1}^{M} \alpha_m \bar{W}^{\pi^m,n}_1 - \tilde{\rho}^{(n)}\right| \leq \sqrt{\frac{H^2\log\left(t\epsilon' {K}\right)}{\epsilon' t\sqrt{K}}} \, \forall n,\\
        &\sum_m \alpha_m=1, \alpha_m\geq \epsilon'\quad \forall m. 
        \end{aligned}
        \label{decomp-opt}
\end{equation}

After learning sufficiently accurate $\{\alpha_m\}$ in the second phase, PRI learns the occupancy measure under the mixed policy defined by $\{\alpha_m\}$ and constructs a Markov policy $\tilde \pi$ based on the learned occupancy measure.

Consider a CMDP with action space $\hat{\cal D}$ such that $\hat{\cal D}_{h,x}\subseteq {\cal A},$ and $\Pi_{\hat{\cal D}}$ be the set of all policies induced from the reduced action space. If an optimal policy for the  CMDP cannot be constructed from the reduced action space, then it implies that 
$$\sigma_{\hat{\cal D}}:=\min_{\pi\in \Pi_{\hat{\cal D}}}\max\left\{\tilde{V}_1^{\pi^*}-V_1^{\pi}, \max_n\left(\tilde{\rho}^{(n)}-W_1^{\pi,n}\right)\right\}>0,$$
i.e., either the reward value function is smaller or one of the constraints has to be violated under any policy on the reduced action space. Note that the minimum is well defined because the set of occupancy measures induced by the policies in $\hat{\cal D}$ is convex. Define 
$$\sigma_{\min}=\min_{\sigma_{\cal D}\not=0} \sigma_{\cal D}.$$ 
In spirit, $\sigma_{\min}$ is equivalent to the reward gap between the best and the second best arm in bandits. We say a CMDP instance is {\em well-separated} if $\sigma_{\min}$ is a positive constant independent of $K.$ 

\begin{theorem} Assume the CMDP is well separated, and Slater's condition holds for the associated LP. Running PRI for $\tilde{K}=20H^2SAK^{0.25}\log K+2K.$ For a sufficiently large $K$, we have
\begin{itemize}[leftmargin=*]
    \item PRI guarantees $\tilde{\cal O}(H\sqrt{K})$ regret and zero constraint violation, and
        \item with probability $1-\tilde{\cal O}\left(\frac{1}{\sqrt{K}}\right)$, PRI yields policy $\tilde{\pi}$ such that 
 with $V_1^{\tilde{\pi}} = V_1^{\pi^*}-\tilde{\cal O}\left(\frac{1}{\sqrt{K}}\right)$ and $W_1^{\tilde{\pi},n} \geq \rho^{(n)}$ for all $n.$ 
The algorithm also has no more than $N$ stochastic decisions. 
\end{itemize}  \label{thm:main}   
\end{theorem}
The proof of this theorem can be found in the next section.

\begin{theorem} Given any online learning algorithm, there exists a CMDP instance that is well-separated and its LP satisfies Slater's condition such that either the regret or the constraint violation is $\Omega(H\sqrt{K})$ under the algorithm. 
\end{theorem}
The proof can be found in Appendix \ref{app:lb} and is based on the example given in \cite{CheGanSal_22} for safe linear bandits.  
We remark that the $\Omega(H\sqrt{HSA K})$ lower bound \cite{JinAllBub_18,DomMenKau_21} for unconstrained MDPs does not apply here because the lower bound is based on an MDP instance that is {\em not} well separated. 

\section{Analysis}

We define the set of optimal policies for the {\em tightened} CMDP as $\Pi^*$ and the subset associated with extreme points as $\Pi^{*,e}$. Furthermore, given an action space $\tilde{\cal D}=\prod_{(h,x)} \tilde{\cal D}_{h,x},$ let $\Pi_{\tilde{\cal D}}$ to be the set of feasible policies given the action space. 

The following theorem shows that after Policy Pruning, the algorithm outputs an action space $\tilde{\cal D}$ such that no more action can be removed without losing the optimality and Pruning phase inccurs at most $20H^2SAK^{0.25}\log K$ regret and constraint violation. 
\begin{theorem} \label{the:multi}
When $K$ is sufficiently large, with probability $1-\tilde{\cal O}(K^{-9/8}),$  Policy Pruning outputs 
$\tilde{\cal D}$ such that $\Pi_{\tilde{\cal D}}\bigcap \Pi^*\not=\emptyset$ but $\Pi_{\tilde{\cal D}'}\bigcap \Pi^*=\emptyset$ for any $\tilde{\cal D}'\subset \tilde{\cal D}.$
In other words, no action can be further removed without losing the optimality.  Furthermore, the regret and constraint violation, with respect to the tightened CMDP, during Policy Pruning are bounded by $20H^2SAK^{0.25}\log K$ with probability one.   
\end{theorem}
More discussions and detailed proof are deferred to Appendix \ref{sub4} due to the page limit. 

The following theorem shows that the regret and constraint violation during the refinement phase. The proof can be found in Appendix \ref{sub2}.

\begin{theorem}[Refinement]
The regret and constraint violation, with respect to the tightened CMDP, during the policy refinement phase are both ${\mathcal{O}}(H\sqrt{K}\log K).$ 
\label{thm:refinement}
\end{theorem}

Since the regret and constraint violations are sublinear, with a high probability, the refinement phase learns a near optimal mixed policy, which is a combination of $M$ greedy policies for $M\leq 2^N.$ In the following theorem, we show that the identification phase is to find a single near-optimal policy by using the occupancy measure of the mixed policy. The proof can be found in Appendix \ref{sub3}.

\begin{theorem}[Identification] \label{thm:ide}
The regret and constraint violation, with respect to the tightened CMDP, during the policy identification phase are $\mathcal{O}(H\sqrt{K}\log K)$. 
Furthermore,   with probability $1-\tilde{\cal O}\left(\frac{1}{\sqrt{K}}\right)$, PRI yields policy $\tilde{\pi}$ such that 
 with $V_1^{\tilde{\pi}} = \tilde{V}_1^*-{\cal O}\left({{\frac{H\log K}{\sqrt{K}}}}\right)$ and $W_1^{\tilde{\pi},n} \geq \rho^{(n)}+\epsilon_\rho-{\cal O}\left({{\frac{H\log K}{\sqrt{K}}}}\right)$ for all $n,$ where $\tilde{V}^*_1$ is the solution to the tightened CMDP. 
The algorithm also has no more than $N$ stochastic decisions. 
\end{theorem}

Theorems \ref{the:multi}-\ref{thm:ide} established the regret and constraint violation bounds under RPI with respect to the tightened CMDP. The following lemma (Lemma 1 in \cite{WeiLiuYin_22}) establishes the connection between the original CMDP and the tightened CMDP. 
 
\begin{lemma}\label{le:conservative}
Denote by $\pi^*$  an optimal policy for the original CMDP and $\tilde{\pi}^*$ an optimal policy for the tightened CMDP. If Slater's condition holds for the LP associated with the original CMDP, we have $V^{\pi^*}_1 - V^{\tilde{\pi}^*}_1 \leq \frac{H\epsilon_\rho}{\delta}$, where $\delta$ is Slater's constant. 
\end{lemma}

\begin{proof}[Proof of Theorem \ref{thm:main}]
By summarizing the results from Theorem \ref{the:multi}-\ref{thm:ide}, we can first conclude that the regret, with respect to the tightened CMDP, is bounded by 
\begin{align*}
&HSAK^{0.25}\log K+{\mathcal{O}}(H\sqrt{K}\log K)+\mathcal{O}(H\sqrt{K}\log K)\\
=&\mathcal{O}(H\sqrt{K}\log K)=\tilde{\mathcal{O}}(H\sqrt{K}).    
\end{align*}
 Based on Lemma \ref{le:conservative}, we can further conclude that the regret with respect to the original CMDP is bounded by 
\begin{align*}
\tilde{\mathcal{O}}(H\sqrt{K})+\tilde{K}\frac{H\epsilon_{\rho}}{\delta}=\tilde{\mathcal{O}}(H\sqrt{K}),    
\end{align*} where $\tilde{K}=20H^2SAK^{0.25}\log K+2K$ is the total number of episodes under PRI. 

Similarly, the constraint violation with respect to the tightened constraints is also $\mathcal{O}(H\sqrt{K}\log K)$ for all $n,$ i.e. 
 $$\tilde{K} \left(\rho^{(n)}+\epsilon_{\rho}\right) - \mathbb{E}\left[\sum_{k=1}^{\tilde{K}} W_1^{\pi_k,n} \right]={\mathcal{O}}(H\sqrt{K}\log K),$$ which implies for a sufficiently large $K,$
 $$\tilde{K}\rho^{(n)}- \mathbb{E}\left[\sum_{k=1}^{\tilde{K}} W_1^{\pi_k,n} \right]={\mathcal{O}}(H\sqrt{K}\log K) - \tilde{K}\epsilon_{\rho}<0,$$
where the last inequality holds due to our choice of $\epsilon_{\rho}=\frac{\log^2 K}{\sqrt{K}}.$ Based on Theorem \ref{thm:main} and Lemma \ref{le:conservative}, we can also conclude that with a high probability, the identified policy satisfies 
 ${V}_1^{\pi^*}-V_1^{\tilde{\pi}} \leq \frac{H\epsilon_{\rho}}{\rho}+\tilde{\cal O}\left(\frac{1}{\sqrt{K}}\right)=\tilde{\cal O}\left(\frac{1}{\sqrt{K}}\right)$ and $W_1^{\tilde{\pi},n} \geq \rho^{(n)}$ for all $n.$

\end{proof}
Note that the order-wise bounds are independent of $S$ and $A,$ unlike those in the literature. However, there is an implicit dependence on $S$ and $A$ as the results hold only when $K$ is sufficiently large and how large $K$ needs to be depends on $S$ and $A.$    

\section{Experiments}
\subsection{ Synthetic CMDP}

This section presents numerical evaluations of the proposed algorithm. We first evaluated our algorithm for a synthetic CMDP with a single constraint. The objective is to maximize the cumulative reward while guaranteeing that the cumulative utility is at least $2.$ Comparison between Triple-Q and PRI can be found in Figure~\ref{fig:cmdp-re}. Experiment details can be found in Appendix \ref{app:sim-mdp}.

We can observe that PRI significantly outperforms Triple-Q on regret. At the end of the $8\times 10^6$ episodes, Triple-Q has a regret of $1.57 \times 10^6$ and constraint violation of $-4.17\times 10^5$. In contrast, the regret and constraint violation under PRI are $ 6.89\times 10^4$ and $-8 \times 10^4$, respectively. Thus, PRI has significantly lower regret than Triple-Q while achieving zero constraint violation. Since the full CMDP model is given, we can obtain the optimal solution by using the linear programming approach. The cumulative reward and cumulative utility we get for our learned policy are $1.561,$ and $2.016,$ which match the optimal solution $1.573$ and $2.$ The learned policy includes a single stochastic decision and can be found in Appendix \ref{app:sim-mdp}.

\begin{figure}[tb]
\centering
   \begin{subfigure}[b]{0.49\textwidth}
         \centering
         \includegraphics[width=8cm,height=5cm]{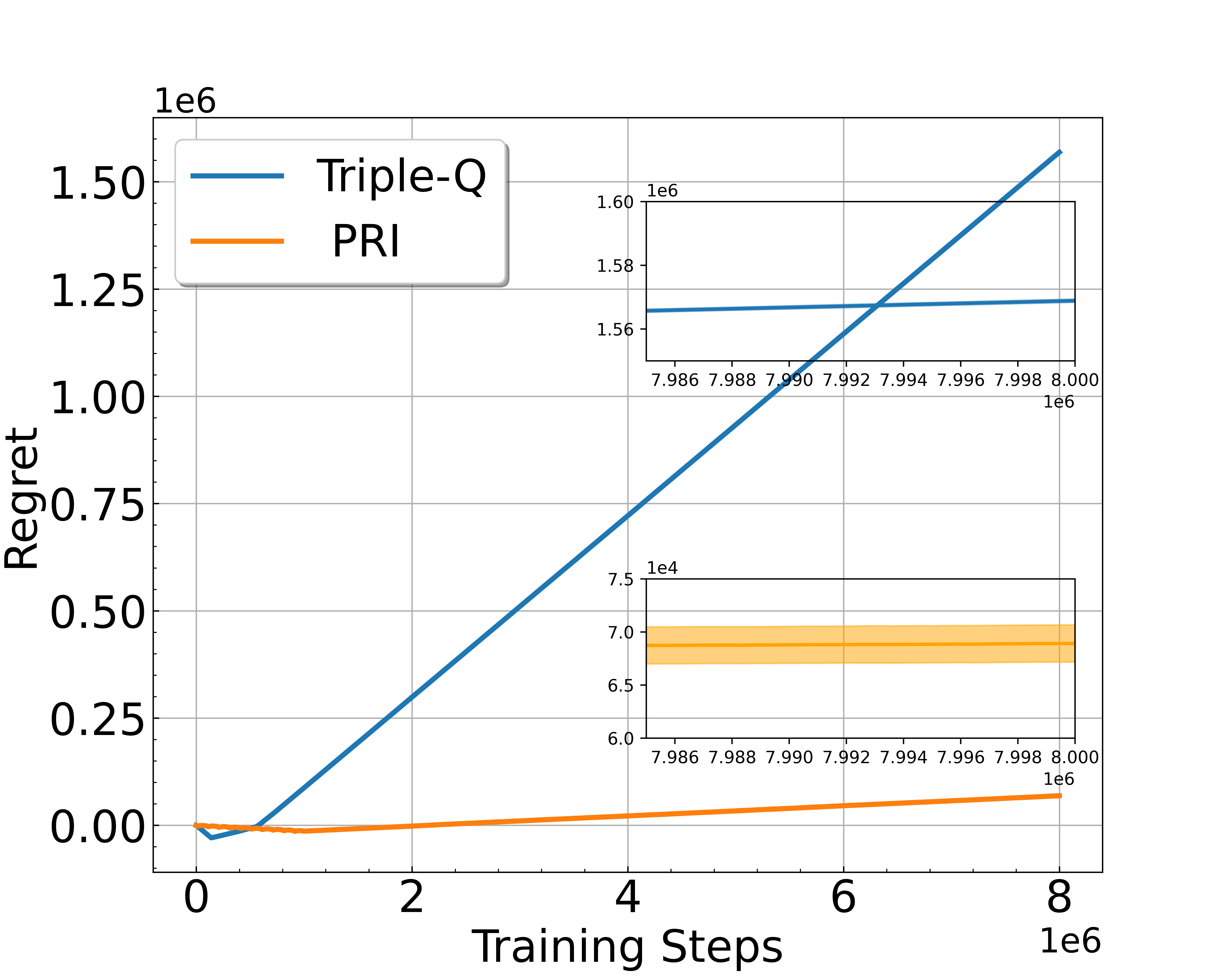}
         \caption{Regret}
         \label{fig:CMDP_r}
     \end{subfigure}
 \hfill
    \begin{subfigure}[b]{0.49\textwidth}
         \centering
         \includegraphics[width=8cm,height=5cm]{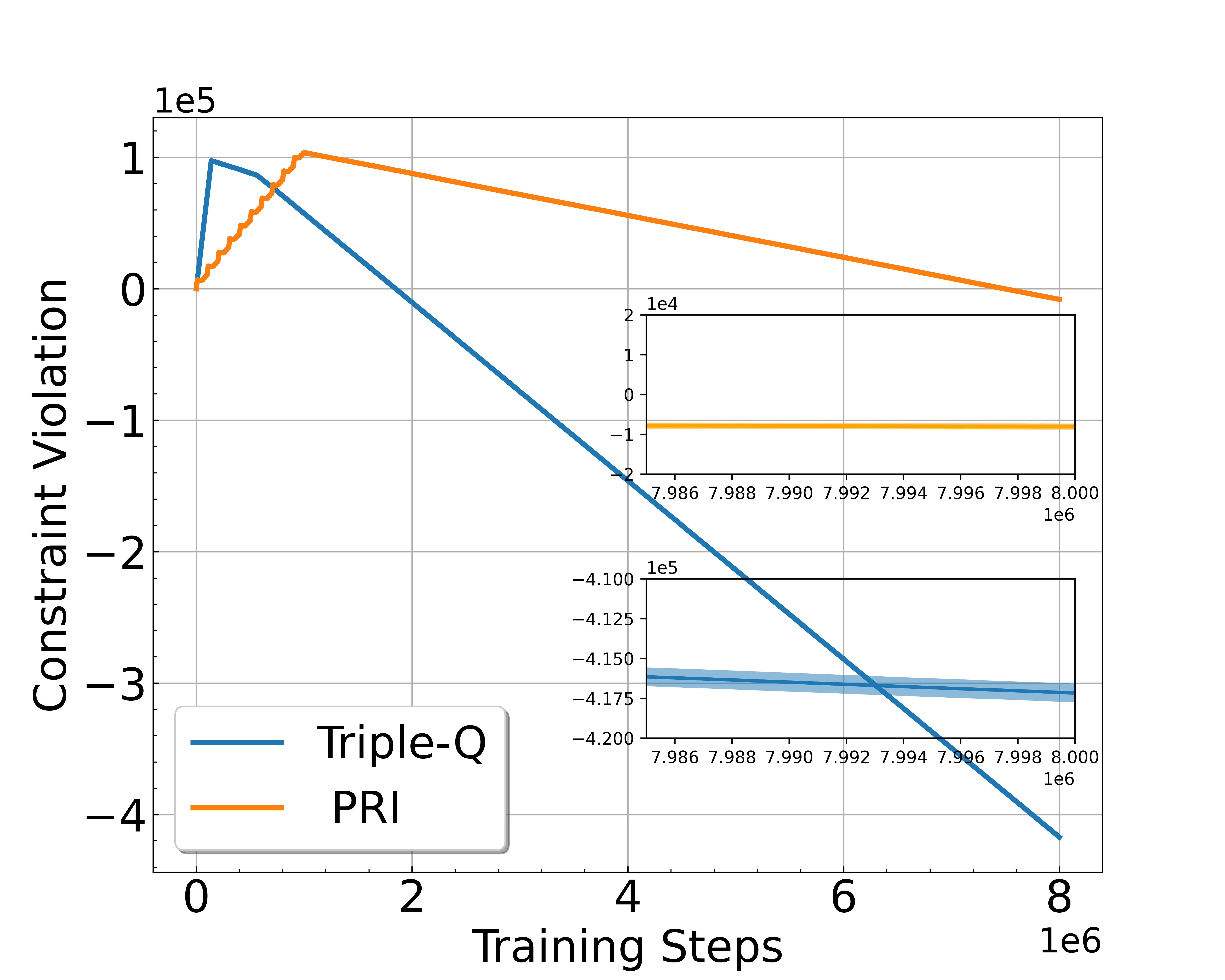}
         \caption{Constraint Violation}
         \label{fig:CMDP_c}
     \end{subfigure}
\caption{Results for a synthetic CMDP with a unique solution, the shaded region represents the 95\% confidence interval.} \label{fig:cmdp-re}
\end{figure}
\subsection{ Grid-world}

In our second experiment, which is a grid-world environment (refer to Appendix \ref{sec:app-grid} for details), we compared Triple-Q with PRI, and the results are shown in Figure \ref{fig;grid-result}. PRI consists of $200,000$ episodes for the initial Triple-Q phase, followed by $200,000$ episodes for each comparing pruning phase(Algorithm 3). Note that in order to achieve a better performance, we add an early stop for each policy comparison, that is, we terminate the comparison if the constraint violation or the regret with the reduced action space becomes too large. Both policy refinement and policy identification phases include $5,000,000$ episodes each. For reference, we ran Triple-Q for the same number of episodes. The results on regret and constraint violation are shown in Figure \ref{fig:2r} and \ref{fig:2c}. We can observe that Triple-Q has a regret of $3.19 \times 10^6$ and a constraint violation of $-5.27 \times 10^5$, whereas PRI achieves $1.69 \times 10^5$ regret and $-6.51 \times 10^3$ constraint violation, indicating substantially lower regret with PRI. For learned policy, the cumulative reward is $5.000$ and the cumulative cost is $0.499$, so the policy satisfies the requirement of cost being $\leq 0.5$. The learned policy takes one stochastic decision and can be found in Appendix \ref{sec:app-grid}. 

\begin{figure}[tb]
\centering
   \begin{subfigure}[b]{0.49\textwidth}
         \centering
         \includegraphics[width=8cm,height=5cm]{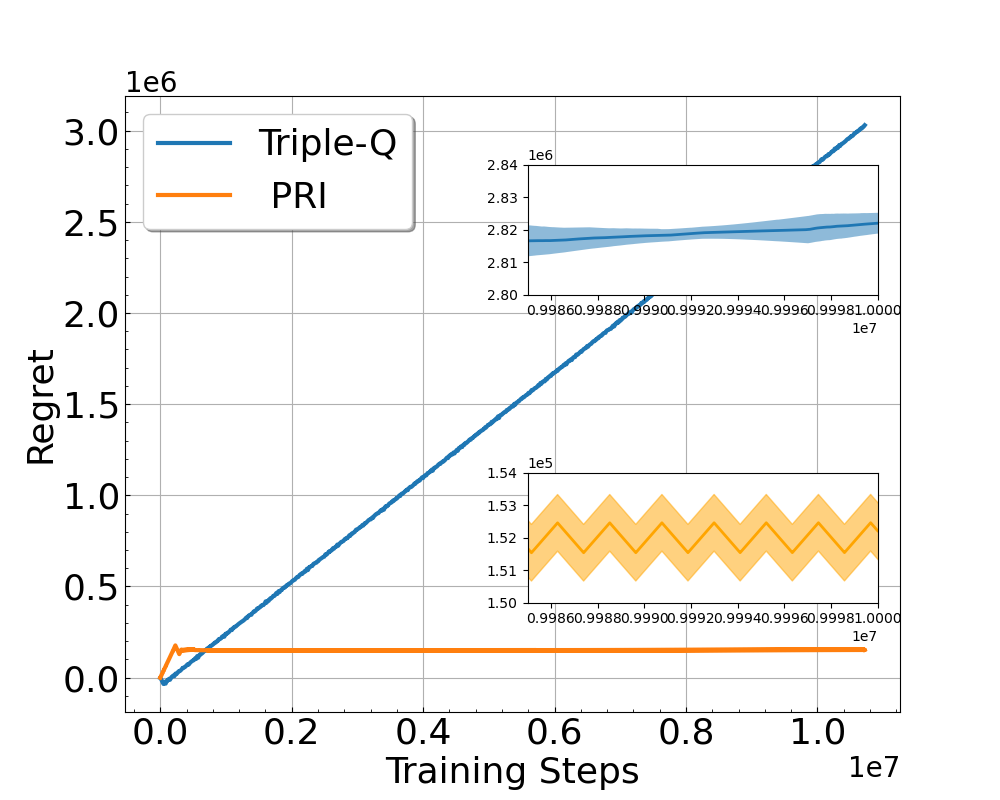}
         \caption{Regret}
         \label{fig:2r}
     \end{subfigure}
 \hfill
    \begin{subfigure}[b]{0.49\textwidth}
         \centering
         \includegraphics[width=8cm,height=5cm]{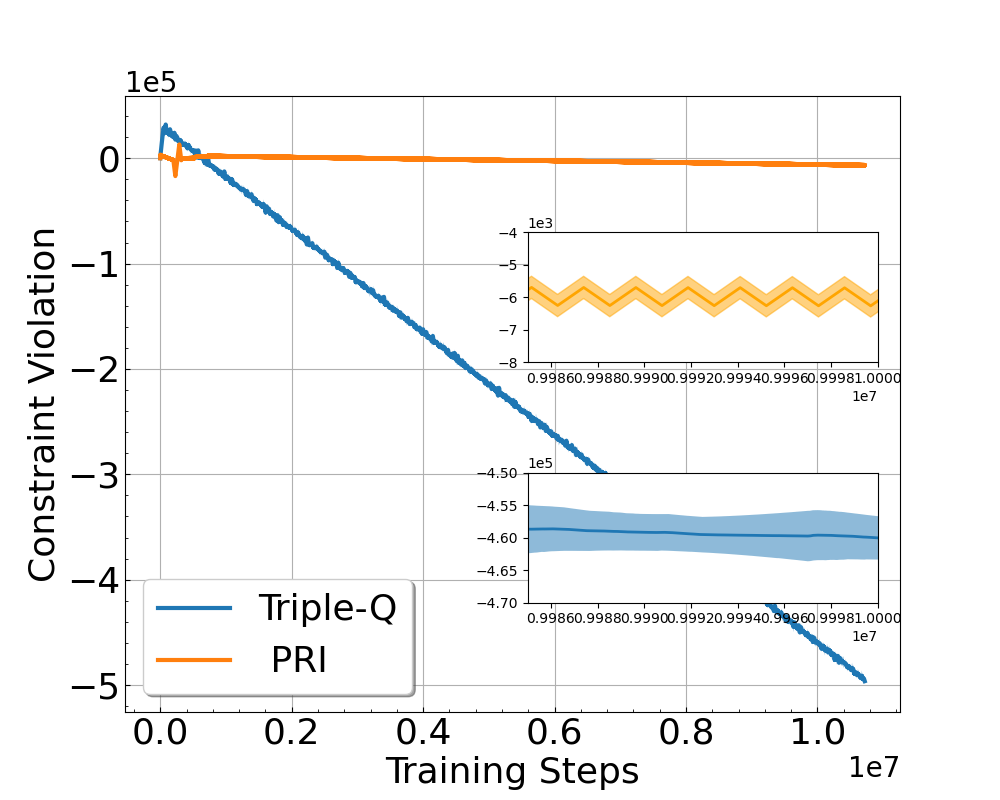}
         \caption{Constraint Violation}
         \label{fig:2c}
     \end{subfigure}
\caption{Results for the grid world environment, the shaded region represents the 95\% confidence interval.} \label{fig;grid-result}
\end{figure}

\section{Conclusions}

In this paper, we developed a model-free, regret-optimal algorithm for online CMDPs, called PRI. The algorithm is based on a fundamental observation that for a CMDP with $N$ constraints, there exists an optimal policy that includes at most $N$ stochastic decisions. In the tabular setting, PRI guarantees $\tilde{\mathcal{O}}(H\sqrt{K})$ regret and zero constraint violation. The regret bound is independent of the size of state and action spaces for sufficiently large $K.$ We can also show that PRI guarantees  $\tilde{\mathcal{O}}(H\sqrt{K})$ regret and $\tilde{\mathcal{O}}(H\sqrt{K})$ constraint violation when the constraint violation cannot be canceled across episodes with some minor modifications of the algorithm. The details can be found in Appendix \ref{appendix:violation}.

\bibliography{refs,inlab-refs}
\bibliographystyle{apalike}

\newpage
\appendix
\onecolumn
\section{NOTATION TABLE}
The most important notations that are repeatedly used are summarized in Table \ref{ta:notations}.
\begin{table}[ht]
	\caption{Notation Table}
	\label{ta:notations}
	\vskip 0.10in
	\begin{center}
            \resizebox{\textwidth}{52mm}{
		\begin{tabular}{c|l}
			\toprule
			Notation & Definition  \\
			\midrule
			$ K $    &   The number  of episodes (the total number of episodes under PRI is $\tilde{K}=20H^2SAK^{0.25}\log K+2K$)\\
			\hline
			$ S$    &   The number of states\\
			\hline
			$ A$    &   The number of actions\\
			\hline
			$ H$    &   The length of each episode\\
                \hline 
                $N$ & The number of constraints\\
			\hline
			$ [H]$    &   Set $\{1,2,\dots,H\}$\\
			\hline
			$Q_{h}^\pi (x,a)$ & The reward Q-function at step $h$ in episode $k$ under policy $\pi$\\
			\hline
			$V_{h}^\pi (x)$ & The value-function at step $h$ in episode $k$ under policy $\pi$\\
			\hline
			$C_{h}^{\pi,n}(x,a)$ & The $n$-th estimated utility Q-function at step $h$ followed by policy $\pi$ \\
			\hline
			$W_{h}^{\pi,n}(x)$ & The $n$-th estimated utility value-function at step $h$ followed by policy $\pi$\\
			\hline
			$r_h(x,a)$ & The reward of (state, action) pair $(x,a)$ at step $h.$ \\
			\hline
			$g_h^{(n)}(x,a)$ & The $n$-th utility of (state, action) pair $(x,a)$ at step $h.$ \\ 
			\hline
			$q_h^*$ & The optimal solution to the LP of the CMDP. \\
			\hline
			${q}_h^{\pi}$ & The corresponding occupancy measure to policy $\pi$ \\
                \hline
                $\epsilon_\rho$ & a small value for tightened constraints\\
                \hline
                $\{\alpha_m\}_{m=1,...M}$ & The probability distribution over a set of $M$ greedy policies to form a mixed policy\\
                \hline
                $\tilde{\cal D}_{h,x}$ & The set of actions that can be taken in state $x$ at step $h$\\
                \hline
                $\Pi^*$ & The set of optimal policies\\
                \hline
                $\Pi_{\tilde{\cal D}}$ & The set of feasible policies given the action space $\tilde{\cal D}$\\
                \hline
                $\sigma_{\cal D}$ & The minimal gap of $V$ and $W^{n}$ between the best policy with reduced action space $\cal D$ and the optimal policy\\
			\hline
			$\delta$ & Slater's constant.\\
			\hline
			$\mathbb{I}(\cdot)$ & The indicator function\\
			\hline
			\bottomrule
		\end{tabular}}
	\end{center}
\end{table}

\section{Review of Triple-Q} \label{sec:ap-qqq}

In this section, we briefly review Triple-Q for CMDPs \cite{WeiLiuYin_22}. The design of Triple-Q is based on the primal-dual approach in optimization. The notions used in this section may be slightly abused, but we will make sure the definitions are clear. Consider the case with only one constraint for simplification. Given a Lagrange multiplier $\lambda,$ we consider the Lagrangian of the problem from a given initial state $x_1:$
\begin{align}
&\max_\pi V_{1}^\pi(x_1)+ \lambda\left(W_{1}^\pi(x_1)-\rho\right) \nonumber \\
=&\max_\pi \mathbb{E}\left[\sum_{h=1}^H r_h(x_h,\pi_h(x_h))+\lambda g_h(x_h,\pi_h(x_h)) \right]-\lambda\rho,\nonumber
\end{align} which is an unconstrained MDP with reward $r_h(x_h,\pi_h(x_h))+\lambda g_h(x_h,\pi_h(x_h))$ at step $h.$ Assuming we solve the unconstrained MDP and obtain the optimal policy, denoted by $\pi^*_\lambda,$ we can then update the dual variable (the Lagrange multiplier) using a gradient method:
\begin{equation}
\lambda \leftarrow \left(\lambda +\rho - \mathbb E\left[W_1^{\pi_\lambda^*}(x_1)\right]\right)^+.
\end{equation} While primal-dual is a standard approach, analyzing the finite-time performance, such as regret or sample complexity, is particularly challenging. Triple-Q is designed as a two-time scale algorithm for addressing the trade-off between regret and constraint violations. 
\begin{itemize}
    \item At each step, Triple-Q updates two the Q-values for $(x_{h-1},a_{h-1})$ after observing $(s_h,a_h),$ reward $r_h(x_h,a_h)$ and utility $g_h(x_h,a_h)$ in a fast time scale. In particular,
    \begin{align*}
        {Q}_{h-1} (x_{h-1},a_{h-1}) \leftarrow (1-\alpha_t)Q_{h-1} (x_{h-1},a_{h-1}) \\
        + \alpha_t\left(r_{h-1}(x_{h-1},a_{h-1})+V_{h} (x_{h})+b_t\right) \\
        {C}_{h-1} (x_{h-1},a_{h-1}) \leftarrow (1-\alpha_t)C_{h-1} (x_{h-1},a_{h-1}) \\
        + \alpha_t\left(g_{h-1}(x_{h-1},a_{h-1})+W_{h}(x_{h})+b_t\right)
    \end{align*}
    \item at the end of each frame, the virtual queue length is updated in a slow time scale manner as $\left(Z +\rho+\epsilon -  \frac{\bar{C}}{K^\alpha}\right)^+,$ where $\bar{C}$ is the average of all the Q-values of the utility function at the initial state-action $(x_1,a_1).$ 
\end{itemize}

The algorithm only needs to know the values of $H,$ $A,$ $S$ and $K,$ and no other problem-specific values are needed.  Furthermore, Triple-Q includes updates of two Q-functions per step: one for $Q_{h}$ and one for $C_{h};$ and one simple virtual queue update per frame. So its computational complexity is similar to SARSA. 

\section{Proofs of the Technical Lemmas} \label{sub5}

\subsection{Proof of Lemma \ref{le:spar} (Limited Stochasticity) } 
\begin{customlemma}{\ref{le:spar}}
If $q^*=\{q^*_{h}(x,a)\}_{h,x,a}$ is an optimal solution to the CMDP problem \eqref{eq:cmdp-offline}-\eqref{eq:proba} and is an extreme point, then there are at most $HS+ N$ nonzero values in $q^*$. This implies that the optimal policy derived from $q^*$ includes at most $N$ stochastic decisions. 
\end{customlemma}
\begin{proof}
The LP has $HSA$ decision variables $\{q_h(s,a)\}$ in total. So at an extreme point, at least $HSA$ constraints become tight. In other words, at least $HSA$ constraints become equalities under solution $q^*$.   
Since there are only $HS+N$ constraints defined in \eqref{eq:constraint}-\eqref{eq:init-distr}, at least $$
HSA-HS-N=HS(A-1) - N$$ constraints in \eqref{eq:proba} become tight (equality) under $q^*.$ Therefore, there are at least $HS(A-1)-N$ zeros in $q^*$ or at most $HS+N$ nonzero values in $q^*$. 

Now suppose the optimal policy obtained from $q^*$ has less than $HS-N$ greedy decisions. Then $q^*$ would have at least $$HS-N-1 + 2(N+1)=HS+N+1$$  nonzero values because each greedy decision requires one nonzero $q_h(x,a)$ and each stochastic decision requires at least two nonzero $q_h(x,a).$ This leads to a contradiction. 
\end{proof}

\subsection{Proof of Lemma \ref{lem:decom} (Decomposition)} \label{sub6}
 \begin{customlemma}{\ref{lem:decom}}
 Given any Markov policy $\pi$ and its corresponding occupancy measure $q,$  there exists a set of $M$ greedy policies and a probability distribution $\{a_m\}_{m=1, \cdots, M}$ such that the mixed policy, which selects a greedy policy $\pi^m$ at the start of an episode with probability $a_m$ and subsequently follows it, has the same occupancy measure $q$ as the original policy $\pi$. 
\end{customlemma}
\begin{proof}
To simplify the notation, we will prove the lemma for the case where $|\tilde{\cal D}_{h,x}(q)|\in\{1,2\},$ i.e., any stochastic decision takes two possible actions, and $a\in\{0, 1\}$.  The extension to the general case is trivial.    

Under a Markov policy $\{\pi_h\}_{h=1}^H$, the actions are independently chosen given state $x$ and step $h.$ Suppose we will execute the Markov policy for $K$ episodes. We will generate $K$ matrices $\{B_k\}_{k=1}^K$ of size $H\times S$ such that $B_k(h,x)$ is a realization of a random variable with distribution $\pi_h(\cdot|x)$. All these values are independently generated.  Now to execute policy $\pi$ at episode $k,$ at state $x$ and step $h,$ the agent takes action $a$ such that $B_k(h,x) = a.$ This is statistically the same as sampling an action using $\pi_h(\cdot|x)$ when reaching state $x$ at step $h.$

We note that each binary matrix $B_k$ corresponds to a greedy policy from the $M$ greedy policies and vice versa. Furthermore, the binary matrix  associated with greedy policy $\pi^m$ is generated with probability 
$$\alpha_m = \prod_{h,x} \left(\sum_{a\in \tilde{\cal D}_{h,x}(q)} \pi_h(a|x)\pi^m_h(a|x)\right),$$ because 
\begin{align*}
    &\sum_{a\in \tilde{\cal D}_{h,x}(q)} \pi_h(a|x)\pi^m_h(a|x)\nonumber \\
    =&\sum_{a\in \tilde{\cal D}_{h,x}(q)} \pi_h(a|x)\mathbb{I}(\pi^m_h(a|x)=1),
\end{align*} which is the probability that action selected by the greedy policy $\pi^m$ is also selected under policy $\pi.$
Therefore, if we consider a mixed policy that chooses policy $\pi^m$ with probability $a_m,$ then it is statistically the same as policy $\pi$ and has the same occupancy measure $q.$ 
\end{proof}

\section{Proof of Theorem \ref{the:multi} (policy pruning)} \label{sub4}

To prove our main result, we first recall the regret and constraint violation guaranteed under Triple-Q \cite{WeiLiuYin_22} in the following lemma.
\begin{lemma} \label{triple-q}
For sufficiently large $K$, over $K$ episodes,  Triple-Q guarantees $\tilde {\cal O} (K^{0.8})$ regret and zero constraint violation, and furthermore, 
    \begin{align}
\Pr\left(    K \tilde{\rho}^{(n)}- \sum_{k=1}^K W_1^{\pi_k,n} \leq 0\right) =1-{\cal O}\left(\frac{1}{K^2}\right).
    \end{align}
\end{lemma}

Noting that we use Triple-Q in the pruning phase, the assumptions for Triple-Q, like Slater's condition, are still required. We also note that if the LP for the original CMDP satisfies Slater's condition, then the tightened CMDP satisfies Slater's condition for large $K$ as well. 

The following lemma states that if $\Pi_{\tilde{\cal D}}$ includes an optimal policy, then Algorithm \ref{alg:comp} can correctly decide whether $\Pi_{\hat{\cal D}}$ includes an optimal policy with a high probability. 
\begin{lemma}
Assume $\Pi_{\tilde{\cal D}}$ includes an optimal policy to the tightened CMDP. Then with probability $1-{\cal O}(K^{-0.02}),$ Algorithm \ref{alg:comp}, Compare($\tilde{\cal D}$,$\hat{\cal D}$) outputs 1 if $\Pi_{\hat{\cal D}}$ includes an optimal policy and outputs 0 otherwise. 
\label{lem:comp}
\end{lemma}
\begin{proof}

Consider the tightened CMDP with action space $\hat{\cal D}$ such that $\hat{\cal D}_{h,x}\subseteq {\cal A},$ and $\Pi_{\hat{\cal D}}$ be the set of all policies induced from the reduced action space. If an optimal policy for the tightened CMDP cannot be constructed from the reduced action space, then it implies that 
$$\sigma_{\hat{\cal D}}:=\min_{\pi\in \Pi_{\hat{\cal D}}}\max\left\{\tilde{V}_1^{\pi^*}-V_1^{\pi}, \max_n\left(\tilde{\rho}^{(n)}-W_1^{\pi,n}\right)\right\}>0,$$
i.e., either the reward value function is smaller or one of the constraints is a violation under any policy on the reduced action space. Note that the minimum is well defined because the set of occupancy measures induced by the policies in $\hat{\cal D}$ is convex. Define 
$$\sigma_{\min}=\min_{\sigma_{\cal D}\not=0} \sigma_{\cal D}.$$ 
We assume $\sigma_{\min}$ for the original CMDP is a positive constant independent of $K,$ which implies that $\sigma_{\min}$ for the tightened problem is a positive constant independent of $K$ as well because it is tightened by $\frac{\log^2 K}{\sqrt{K}},$ i.e. $\tilde{\cal O}(1/\sqrt{K})$ perturbation to the constraints. 

We first consider the case that $\Pi_{\hat{\cal D}}$ includes an optimal policy. In this case, we will show that $|v^*-\tilde{v}|\leq \frac{2}{K^{0.03}}$ with a high probability. 
Define $$\bar{V}_1=\frac{1}{K^{0.25}} \sum_{k=1}^{K^{0.25}}V_1^{\pi_k},$$
where $\pi_k$ is the policy used in the $k$th episode with action space $\hat{\cal D}$.
Similarly, define $$\bar{V}^*_1=\frac{1}{K^{0.25}} \sum_{k=1}^{K^{0.25}}V_1^{\pi_k},$$ 
where $\pi_k$ is the policy used in the $k$th episode with action space $\tilde{\cal D}$.   

Note that \begin{align*}
v^* - \tilde{v}=(v^*-  \bar{V}^{*}_1)+ (\bar{V}^{*}_1-{V}^{\tilde{\pi}^*}_1)+({V}^{\tilde{\pi}^*}_1-\bar{V}_1)+(\bar{V}_1-\tilde{v}).   
\end{align*} We next bound the four terms individually. 

Let $v_{k,1}$ be the cumulative reward received in episode $k$ with the reduced action space and $V^{\pi_k}_1$ be the reward value function. Note that 
$$X_\tau=\sum_{k=1}^\tau \left(v_{k,1}-V_1^{\pi_k}\right)$$ is a Martingale. By Azuma's inequality, we have
\begin{align}
    \Pr \left(\left|\tilde{v} - \bar{V}_1\right| \leq \sqrt{\frac{2H^2\log K^{0.25}}{K^{0.25}}} \right)  \geq 1 - \frac{1}{2K^{0.25}}.\label{eq:azu-1}
\end{align} A similar argument yields that 
\begin{align}
    \Pr \left(\left|v^* - \bar{V}_1^{*}\right| \leq \sqrt{\frac{2H^2\log K^{0.25}}{K^{0.25}}} \right)  \geq 1 - \frac{1}{2K^{0.25}}. \label{eq:azu-2}
\end{align}

We next bound $|\bar{V}^{*}_1-{V}^{\tilde{\pi}^*}_1|$ based on Lemma \ref{triple-q} and the Markov inequality.  First, based on the Markov inequality, we have 
\begin{align}
\Pr\left(|\bar{V}^{*}_1-{V}^{\tilde{\pi}^*}_1|\geq K^{-0.03}\right)
\leq  \frac{\mathbb{E}\left[ |\bar{V}^{*}_1-{V}^{\tilde{\pi}^*}_1|\right]}{K^{-0.03}}
\end{align}
Note that we have
\begin{align}
\mathbb{E}\left[ |\bar{V}^{*}_1-{V}^{\tilde{\pi}^*}_1|\right]=    &\mathbb{E}\left[\left.V^{\tilde{\pi}^*}_1-\bar{V}^*_1\right|V^{\tilde{\pi}^*}_1\geq \bar{V}^*_1\right]\Pr\left(V^{\tilde{\pi}^*}_1\geq \bar{V}^*_1\right) +
\mathbb{E}\left[\left.\bar{V}^*_1-V^{\tilde{\pi}^*}_1\right|V^{\tilde{\pi}^*}_1< \bar{V}^*_1\right]\Pr\left(V^{\tilde{\pi}^*}_1< \bar{V}^*_1\right) \nonumber\\
=&\mathbb{E}\left[V^{\tilde{\pi}^*}_1-\bar{V}^*_1\right] +
2\mathbb{E}\left[\left.\bar{V}^*_1-V^{\tilde{\pi}^*}_1\right|V^{\tilde{\pi}^*}_1< \bar{V}^*_1\right]\Pr\left(V^{\tilde{\pi}^*}_1< \bar{V}^*_1\right) \nonumber\\
\leq &{\frac{c_1K^{0.2}}{K^{0.25}} +2H \Pr \left(V^{\tilde{\pi}^*}_1 < \bar{V}_1\right)}.
\end{align}
Note that when the constraints are satisfied, we have $V^{\tilde{\pi}^*}_1 \geq \bar{V}^*_1.$ Therefore, according Lemma \ref{triple-q}, $\Pr \left(V^{\tilde{\pi}^*}_1 < \bar{V}^*_1\right)={\cal O}(K^{-0.5})$ (note that Triple-Q is used for $K^{0.25}$ episodes instead of $K$ episodes), which implies that  
\begin{align}
\Pr\left(|V^{\tilde{\pi}^*}_1-\bar{V}^*_1|\geq K^{-0.03} \right) ={\cal O}\left(K^{-0.02}\right).
\end{align}
Similarly, we have 
\begin{align}
\Pr\left(|V^{\tilde{\pi}^*}_1-\bar{V}_1|\geq K^{-0.03} \right) ={\cal O}\left(K^{-0.02}\right).
\end{align}

Combining the inequality above with inequalities \eqref{eq:azu-1} and \eqref{eq:azu-2}, we can conclude that 
\begin{align}
\Pr\left(|v^*-\tilde{v}|\leq 4K^{-0.03} \right) =1-{\cal O}\left(K^{-0.02}\right).
\end{align} Based on Lemma \ref{triple-q}'s high probability result on constraint violation and Azuma's inequality for Martingale, we obtain 
\begin{align}
\Pr\left(|v^*-\tilde{v}|\leq 4K^{-0.03}, \tilde{w}^n \geq \tilde{\rho}^{(n)}\ \forall n \right) =1-{\cal O}\left(K^{-0.02}\right).
\end{align}

Now consider the case no optimal policy exists in $\Pi_{\hat{\cal D}}.$ 
Let $\pi''$ be an optimal policy with action space $\hat{\cal D}$
 and suppose all constraints are satisfied under $\pi''.$ Note that $\pi''$ is {\em not} an optimal policy for the tightened CMDP. We first have 
\begin{align*}
v^* - \tilde{v}= 
v^*-  \bar{V}^{*}_1+ \bar{V}^{*}_1-{V}^{\tilde{\pi}^*}_1+{V}^{\tilde{\pi}^*}_1- V^{\pi''}_1+ V^{\pi''}_1-\bar{V}_1+\bar{V}_1-\tilde{v}
\end{align*}
       
Note that ${V}^{\tilde{\pi}^*}_1- V^{\pi''}_1\geq \sigma_{\min}$ because all constraints are satisfied under $\pi''$ under our assumption. 
Similar to the analysis above, we can also obtain 
\begin{align}
\Pr\left(|V^{\tilde{\pi}''}_1-\bar{V}_1|\geq K^{-0.03} \right) ={\cal O}\left(K^{-0.02}\right).
\end{align}
Summarizing the results above,  using the union bound, we conclude that with probability $1-\mathcal{O}\left(K^{-0.02}\right),$ we have
\begin{align*}
v^* - \tilde{v} \geq \sigma_{\min} - 2K^{-0.03}-2\sqrt{\frac{2H^2\log K^{0.25}}{K^{0.25}}}  > \frac{\sigma_{\min}}{2}\geq 4K^{-0.03}   
\end{align*} for sufficiently large $K$ if an optimal policy is not retained. If none of the policies in $\Pi_{\hat{\cal D}}$ can satisfy the constraints, then it can be easily shown that $\tilde{w}^n < \rho^{(n)}$ with probability $1-{\cal O}(K^{-0.25})$ for some $n$. 

In summary, with probability $1 - \mathcal{O}(K^{-0.02})$, we can correctly verify whether an optimal policy in $\hat{\cal D}.$
\end{proof}

\begin{customthm}{2}
When $K$ is sufficiently large, with probability $1-\tilde{\cal O}(K^{-9/8}),$  Policy Pruning outputs 
$\tilde{\cal D}$ such that $\Pi_{\tilde{\cal D}}\bigcap \Pi^*\not=\emptyset$ but $\Pi_{\tilde{\cal D}'}\bigcap \Pi^*=\emptyset$ for any $\tilde{\cal D}'\subset \tilde{\cal D}.$
In other words, no action can be further removed without losing the optimality.  Furthermore, the regret and constraint violation, with respect to the tightened CMDP, during Policy Pruning are bounded by $20H^2SAK^{0.25}\log K$ with probability one.
\end{customthm}

\begin{proof} 
We first analyze the first $\log K$ rounds of running Triple-Q, each round including $K^{0.25}$ episodes. Without loss of generality, consider the first round and $(h',x',a')$ such that $N_{h'}(x',a')< K^{0.2}.$  Now define the set of policies $$\Pi=\left\{\pi_k: \pi_{k,h'}(a'|x')=0\right\},$$ which includes at least $K^{0.25}-K^{0.2}$ policies. None of these policies uses action $a'$ in state $x'$ at step $h'.$ Furthermore, define action space 
$\hat{\cal D}$ such that $\hat{\cal D}_{h,x}={\cal A}$ and $\hat{\cal D}_{h', x'}={\cal A}\setminus \{a'\}.$ Note that $\Pi\subseteq \Pi_{\hat{\cal D}_{h', x'}}.$

Now suppose that $\Pi_{\hat{\cal D}}$ {\em does not} include an optimal policy, then
$\sigma_{\tilde{\cal D}}>0.$ Now consider a mixed policy $\tilde{\pi}$ that sample a policy uniformly at random from $\Pi,$ so $\tilde{\pi}\in \Pi_{\hat{\cal D}}.$

Define $$\bar{V}^*_1=\frac{1}{K^{0.25}} \sum_{k=1}^{K^{0.25}}V_1^{\pi_k}$$ and  
$$\bar{W}^{*,n}_1=\frac{1}{K^{0.25}} \sum_{k=1}^{K^{0.25}}W_1^{\pi_k,n}$$
where $\pi_k$ is the policy used in the $k$th episode under Triple-Q. According to the proof of Lemma \ref{lem:comp} and  the high probability result in Lemma \ref{triple-q}, we have shown that 
\begin{align}
\Pr\left(|V^{\tilde{\pi}^*}_1-\bar{V}^*_1|\leq K^{-0.03}, \tilde{\rho}^{(n)}-  \bar{W}^{*,n}_1\leq 0 \ \forall n\right) =1-{\cal O}\left(K^{-0.02}\right), 
\end{align} which implies that 
\begin{align}
\Pr\left(|V^{\tilde{\pi}^*}_1-{V}^{\tilde{\pi}}_1|\leq K^{-0.03}+HK^{-0.05}, \tilde{\rho}^{(n)}-  \bar{W}^{\tilde{\pi},n}_1\leq  HK^{-0.05}\ \forall n\right) =1-{\cal O}\left(K^{-0.02}\right). 
\end{align}  The inequality holds because the difference of value functions of different policies are bounded by $H.$ The inequality above means that for a sufficiently large $K,$ $$\sigma_{\tilde{\pi}}\leq K^{-0.03}+HK^{-0.05}<\sigma_{\tilde{\cal D}},$$ which leads to the contradiction. We, therefore, can conclude that with probability $1-{\cal O}\left(K^{-0.02}\right),$ $N_{h'}(x',a')\geq K^{0.2}$ if $\Pi_{\hat D}$ does not include an optimal policy. 
 
To improve the probability, the algorithm repeats this process for $4\log K$ rounds and makes a decision on removing $a'$ or not based on the majority rule. Let $Z_k$ denote the decision in the $k$th episode such that $Z_k=1$ if $N_{h'}(x',a')\leq  K^{0.2}$ and $Z_k=0$ otherwise. So $Z_k$ are i.i.d. random variables and 
 $$z_{h'}(x', a')=\sum_{k=1}^{4\log K} Z_k.$$  Consider the case $\{a'\}$ cannot be removed. In this case, $\mathbb{E}[Z_k]=\Pr\left(Z_k=1\right)=\mathcal{O}(K^{-0.02}).$
Consider a sufficiently large $K$ such that $\Pr\left(Z_k=1\right)\leq \frac{1}{8}.$ According to Hoeffding's inequality,
    \begin{align}
     &\Pr\left(z\geq 2{\log K}\right)\\
     =&\Pr\left(z-\mathbb{E}[z]\geq 2{\log K}-\mathbb{E}[z]\right)\\
     \leq &\Pr\left(z-\mathbb{E}[z]\geq \frac{3\log K}{2}\right)\\
     \leq& e^{-\frac{9}{8}\log K}=\frac{1}{K^{\frac{9}{8}}}. 
     \end{align} So in this case, $a'$ is not be removed with probability $1-K^{-\frac{9}{8}}.$ 
Using the union bound, we can conclude that after the $4\log K$ rounds of Policy Pruning, $\Pi_{\tilde{\cal D}}$ includes an optimality policy with probability $1-{\cal O}(K^{-\frac{9}{8}}).$

After the initial $4\log K$ rounds, the algorithm checks the remaining actions one by one to see whether the action can be removed by using Algorithm \ref{alg:comp}. Note that according to Lemma \ref{lem:comp}, Algorithm \ref{alg:comp} makes a correct decision with probability $1-{\cal O}\left(K^{-0.02}\right).$ Following the analysis above, with $4\log K$ comparisons, the decision on $a'$ is correctly made with a probability of at least $1-{K^{-\frac{9}{8}}}.$ If action $a'$ cannot be removed, the algorithm checks whether using $a'$ alone is sufficient by using Algorithm \ref{alg:comp} $4\log K$ times. Again, the decision can be correctly with a probability at least $1-{K^{-\frac{9}{8}}}$. 

After the algorithm goes through all actions, with probability $1-{\cal O}(K^{-\frac{9}{8}}),$ we obtain action space $$\otimes_{(h,x)}\tilde{\cal D}_{h,x}$$ such that none of the stochastic decision can be reduced to a greedy decision without losing optimality. 
It is easy to verify that the regret and constraint violation of multiple solution pruning are bounded by $20H^2SAK^{0.25}\log K$ because it takes at most $20HSAK^{0.25}\log K$ episodes to finish this phase. 
\end{proof}

\section{Proof of Theorem \ref{thm:refinement} (Refinement)}\label{sub2}
\begin{customthm}{3}
The regret and constraint violation, with respect to the tightened CMDP, during the policy refinement phase, are both ${\cal O}(H\sqrt{K}\log K).$ 
\end{customthm}
\begin{proof}
Consider the action space $\tilde{\cal D}$ after Policy Pruning and the LP for the tightened CMDP with action space $\tilde{\cal D}.$  From Theorem \ref{the:multi}, we have that with probability $1-\tilde{\cal O}(K^{-\frac{9}{8}})$, $\Pi_{\tilde{\cal D}}\bigcap \Pi^*\not=\emptyset.$ Assume it occurs, let $\pi_q\in \Pi_{\tilde{\cal D}}\bigcap \Pi^{*,e}.$ Note that $q$ exists because one of the optimal solutions has to be an extreme point. This also means that 
${\cal D}(q)=\tilde{\cal D}$ because otherwise $\Pi_{{\cal D}(q)}\bigcap \Pi^*=\emptyset$ according to Theorem \ref{the:multi}, which contradicts the fact that $\pi_q$ is an optimal policy. Now from Lemma \ref{le:spar}, we conclude that 
$\tilde{\cal D}$ allows at most $N$ stochastic decisions, and furthermore, at most $2^N$ distinct policies can be constructed from $\tilde{\cal D}.$

We first assume the high probability event in Theorem \ref{the:multi} occurs. Recall that in Lemma \ref{lem:decom}, we have shown that there exists a mixed policy of $M\leq 2^N$ greedy policies defined by $\tilde{\cal D}_{h,x}$ that has the same occupancy measure as that under the optimal policy (for the tightened problem), and $\{\alpha_m^*\}$ are the associated weights.  

Recall that the policy refinement consists of $\sqrt{K}$ rounds. Let $\{\alpha_{t,m}\}_{m=1, \cdots, M}$ be the weights used in round $t.$  Then in the $t$th round, greedy policy $\pi^m$ is used for $\alpha_{t,m}\sqrt{K}$ episodes, where $\alpha_{t,m}$ is the optimal solution to Decomposition-Opt \eqref{decomp-opt}.

First, we will bound the estimation errors of the reward and utility value functions. Recall that PRI uses  $\epsilon' \sqrt{K}$ episodes in each round to estimate the reward value function and the utility value functions instead of all episodes because $\{\alpha_{t,m}\}$ are random variables correlated with the estimated value functions from the previous round. At the end of round $t,$ we have $t\epsilon' \sqrt{K}$ samples from the previous round. Indexing the samples by $k',$ we have  

\begin{align}
    \bar{W}^{\pi^m,n}_1 = \frac{ \sum_{k' = 1}^{t\epsilon' \sqrt{K} }W^{\pi^m,n}_{k',1}}{t\epsilon' \sqrt{K}}.
\end{align}
Define
\begin{align}
    \delta W_1^{\pi^m, n} &=\bar{W}^{\pi^m,n}_1-{W}^{\pi^m,n}_1 \\
    &=\frac{ \sum_{k' = 1}^{t\epsilon' \sqrt{K} }\left(W^{\pi^m,n}_{k',1}-W_1^{\pi^m, n}\right)}{t\epsilon' \sqrt{K}}. \label{eq13}
\end{align}

Since $W^{\pi^m,n}_{k',1}\in[0, H]$ are i.i.d. random variables, by the Azuma-Hoeffding inequality,  we have

\begin{align}
   &\Pr\left( \left|\delta W_1^{\pi^m, n}\right|  \leq \sqrt{\frac{2H^2\log\left(t\epsilon' \sqrt{K}\right)}{\epsilon' t\sqrt{K}}}\right)\\
   \geq &1-\frac{2}{t\epsilon' \sqrt{K}}.\label{eq7}
\end{align}
Similarly, defining $$\delta V_1^{\pi^m} = \bar{V}^{\pi^m}_1-{V}^{\pi^m}_1 =\frac{ \sum_{k' = 1}^{t\epsilon' K^\alpha }\left(V^{\pi^m}_{k',1}-V_1^{\pi^m}\right)}{t\epsilon' \sqrt{K}},$$ we have 
\begin{align}
   &\Pr\left( \left|\delta V_1^{\pi^m}\right|  \leq \sqrt{\frac{2H^2\log\left(t\epsilon' \sqrt{K}\right)}{\epsilon' t\sqrt{K}}}\right)\\
   \geq &1-\frac{2}{t\epsilon' \sqrt{K}}. \label{eq:high-prob-regret}
\end{align}

Therefore, with probability at least $1-\frac{2M}{t\epsilon' \sqrt{K}},$
\begin{align*}
    &\left|\sum_{m=1}^{M} \alpha_m^* \bar{W}^{\pi_m, n}_1 - \tilde{\rho}^{(n)}\right|\\
    =&\left|\sum_{m=1}^{M} \alpha_m^* \left({W}^{\pi_m, n}_1 +\delta {W}^{\pi_m, n}_1\right)- \tilde{\rho}^{(n)}\right|\\  \leq &\sum_{m=1}^{M} \alpha_m^* \left| \delta W^{\pi_m, n}_1 \right|\\
     \leq& \sqrt{\frac{2H^2\log\left(t\epsilon' \sqrt{K}\right)}{\epsilon' t\sqrt{K}}}
\end{align*}
In other words, $\{a_m^*\}$ is a feasible solution to Decomposition-Opt \eqref{decomp-opt} with a high probability, which implies that Decomposition-Opt \eqref{decomp-opt} has a solution with a high probability. 

We now consider $\{a_{t,m}\}$ and the regret and constraint violation in round $t.$ If $\{a_m^*\}$ is a feasible solution to Decomposition-Opt \eqref{decomp-opt}, then   
\begin{align*}
    &\sum_{m=1}^{M} \alpha_{t,m} \sqrt{K} V^{\pi^m}_1\\
    =&\sum_{m=1}^{M} \alpha_{t,m} \sqrt{K} \left(\bar{V}^{\pi^m}_1-\delta{V}^{\pi^m}_1\right)\\
    \geq &\sqrt{K} \left(\sum_{m=1}^{M} \alpha_{t,m} \bar{V}^{\pi^m}_1\right)- \sqrt{K} \max_m |\delta{V}^{\pi^m}_1|\\
    \geq_{(a)} &\sqrt{K} \left(\sum_{m=1}^{M} \alpha_{m}^* \left({V}^{\pi^m}_1+\delta{V}^{\pi^m}_1\right)\right)- \sqrt{K} \max_m |\delta{V}^{\pi^m}_1|\\
    \geq &\sqrt{K} \left(\sum_{m=1}^{M} \alpha_{m}^* {V}^{\pi^m}_1\right)- 2\sqrt{K} \max_m |\delta{V}^{\pi^m}_1|\\
    =& \sqrt{K} V^{\pi^*}_1- 2\sqrt{K} \max_m |\delta{V}^{\pi^m}_1|\\
    \geq &\sqrt{K}\left(V^{\pi^*}_1 - 2\sqrt{\frac{2H^2\log\left(t\epsilon' \sqrt{K}\right)}{\epsilon' t\sqrt{K}}}\right),
\end{align*} where $(a)$ holds because $\{a_{t,m}\}_m$ is the optimal solution to Decomposition-Opt \eqref{decomp-opt}.
In other words, with a high probability, the regret is bounded by 
\begin{align}
2\sqrt{\frac{2\sqrt{K}H^2 \log \left( t\sqrt{K}\right)}{\epsilon' t}} .\label{eq8}
\end{align}
Thus, with probability 
\begin{align}
\prod_{t = 2}^{\sqrt{K}}\left(1 - \frac{2}{t\epsilon' \sqrt{K}}\right)\geq 1 - \frac{2\log K}{\epsilon' \sqrt{K}}
\end{align}
regret in round $t$ is bounded by \eqref{eq8} for all $t.$ Therefore, the regret during policy refinement is bounded by 
\begin{align*}
  &2H\sqrt{K}+\sum_{t=2}^{\sqrt{K}-1} 2\sqrt{\frac{2\sqrt{K}H^2 \log \left( t\sqrt{K}\right)}{\epsilon' t}} \\
  \leq &2H\sqrt{K}+2\sqrt{\frac{2KH^2\log K}{\epsilon'}} \sum_{t=2}^{\sqrt{K}-1} \sqrt{\frac{1}{t\sqrt{K}}}\\
    \leq &2H\sqrt{K}+2\sqrt{\frac{2\sqrt{K}H^2\log K}{\epsilon'}} \int_{t = 1}^{\sqrt{K}} \sqrt{\frac{1}{t}}dt\\
    \leq &2H\sqrt{K}+2\sqrt{\frac{2KH^2\log K}{\epsilon'}} \\
    =&{\cal O}(H\sqrt{K}\log K).
\end{align*} Recall that $\epsilon'=\frac{1}{\log K}.$
The analysis is the same for the constraint violation. We now have shown that during the refinement phase, if the high probability event in Theorem \ref{the:multi} occurs, then the regret and constraint violation, with respect to the tightened CMDP, are both ${\cal O}(H\sqrt{K}\log K)$. If the high probability event does not occur, the regret and constraint violation are bounded by $HK$ since the refinement phase includes $K$ episodes. Combining the two cases together, we have that the regret and constraint violation are bounded by 
$$(1-\tilde{\cal O}(K^{-\frac{9}{8}}))\times {\cal O}(H\sqrt{K}\log K)+\tilde{\cal O}(K^{-\frac{9}{8}})\times HK={\cal O}(H\sqrt{K}\log K).$$
\end{proof}

\section{Proof of Theorem \ref{thm:ide} (Identification)}\label{sub3}
\begin{customthm}{4}
The regret and constraint violation, with respect to the tightened CMDP, during the policy identification phase are $\mathcal{O}(H\sqrt{K}\log K)$. 
Furthermore,   with probability $1-\tilde{\cal O}\left(\frac{1}{\sqrt{K}}\right)$, PRI yields policy $\tilde{\pi}$ such that 
 with $V_1^{\tilde{\pi}} = \tilde{V}_1^*-{\cal O}\left({{\frac{H\log K}{\sqrt{K}}}}\right)$ and $W_1^{\tilde{\pi},n} \geq \rho^{(n)}+\epsilon_\rho-{\cal O}\left({{\frac{H\log K}{\sqrt{K}}}}\right)$ for all $n,$ where $\tilde{V}^*_1$ is the solution to the tightened CMDP. 
The algorithm also has no more than $N$ stochastic decisions. \end{customthm}
\begin{proof}
First, consider the case such that the high probability event defined in Theorem \ref{the:multi} occurs. Consider the $\{a_m\}$ obtained at the end of the refinement phase, and the mixed policy $\hat\pi$ defined by $\{a_m\}.$ According to the proof of Theorem \ref{thm:refinement}, we have with probability $1-{\cal O}(K^{-1}),$   
\begin{align}
    V_1^{\hat \pi}&=\sum_{m=1}^{M} \alpha_{m} V^{\pi^m}_1 \nonumber \\
    &\geq \left(V^{\tilde{\pi}^*}_1 - 2\sqrt{\frac{2H^2\log\left(\epsilon' {K}\right)}{\epsilon' K}}\right)\label{eq:phase3-v}\\
        W_1^{\hat \pi,n}&=\sum_{m=1}^{M} \alpha_{m} W^{\pi^m,n}_1 \nonumber \\
    &\geq \left(W^{\tilde{\pi}^*,n}_1 - 2\sqrt{\frac{2H^2\log\left(\epsilon' {K}\right)}{\epsilon' K}}\right) \quad \forall n.\label{eq:phase3-w}
\end{align} Therefore, the regret and constraint violation during the identification phase, which includes $K$ episodes, are both bounded by 
$${\cal O}(H\sqrt{K}\log K)\times (1-{\cal O}(K^{-1}))\times (1-{\cal O}(K^{-\frac{9}{8}})) + K\times \left(1- (1-{\cal O}(K^{-1}))\times (1-{\cal O}(K^{-\frac{9}{8}})\right)={\cal O}(H\sqrt{K}\log K).$$

For any $(h,x,a)$ such that $0<q^{\pi^m}_h(x,a),$
$$\mathbb{E}\left[\sum_{k=1}^{a_m K} \mathbb{I}(x_{k,h}=x, a_{k,h}=a)\right]=\alpha_m q^{\pi^m}_h(x,a)K,$$ which implies that 
\begin{align*}
    &\Pr\left(\left|{\sum_{k=1}^{a_m K} \mathbb{I}_{(x_{k,h}=x, a_{k,h}=a)}}-\alpha_mq^{\pi^m}_h(x,a)K\right|\leq \sqrt{K\log K} \right)\\
    &= 1-{\cal O}\left(\frac{1}{K}\right)
\end{align*} according to the Azuma-Hoeffding inequality. Define event
\begin{align}
    \Phi = \left\{\left|\frac{1}{K}\sum_{k=1}^{K} \mathbb{I}_{((x_{k,h},a_{k,h})=(x,a))}-\sum_m \alpha_mq^{\pi^m}_h(x,a)\right|\right.\nonumber \\
    \left.\leq M\sqrt{\frac{\log K}{K}}\right\}.
\end{align} We have
\begin{align*}
&\Pr\left(\Phi \right) = 1-{\cal O}\left(\frac{1}{K}\right).
\end{align*} 
Define $\tilde{q}_h(x,a)=\frac{N_h(x,a)}{K},$ which is the empirical occupant measure under policy $\hat{\pi}.$ Note that $$N_h(x,a)={{\sum_{k=1}^{K} \mathbb{I}(x_{k,h}=x, a_{k,h}=a)}}$$ and  $$q^{\hat \pi}_h(x,a)=\sum_m \alpha_mq^{\pi^m}_h(x,a).$$ Therefore, we have with probability $1-{\cal O}(1/K),$
$$\|\tilde{q} -q^{\hat{\pi}}\|_\infty={\cal O}\left(\sqrt{\frac{\log K}{K}}\right),$$ which implies that 
$$\|\tilde{\pi}-\hat{\pi}\|_\infty={\cal O}\left(\sqrt{\frac{\log K}{K}}\right)$$ where 
$$\tilde{\pi}_h(a|x)=\frac{\tilde{q}_h(x,a)}{\sum_{a\in \tilde{\cal D}(h,x)} \tilde{q}_h(x,a)}.$$ We can then  furthermore conclude that  
$$\|{q}^{\hat{\pi}} -q^{\hat{\pi}}\|_\infty={\cal O}\left(\sqrt{\frac{\log K}{K}}\right),$$ and 
$$|V_1^{\tilde{\pi}}-V_1^{\hat{\pi}}|={\cal O}\left(\sqrt{\frac{\log K}{K}}\right)\quad\hbox{and}\quad |W_1^{\tilde{\pi},n}-W_1^{\hat{\pi},n}|={\cal O}\left(\sqrt{\frac{\log K}{K}}\right) \ \forall n$$ which holds because the value functions are the linear functions of the occupancy measure. The theorem holds by combining the results above with inequalities (\ref{eq:phase3-v}) and (\ref{eq:phase3-w}).
\end{proof}

\section{Extension to Constraint Violation without Episode-Wise Cancellation}
\label{appendix:violation}
While the constraint violation defined in (\ref{eq:violation}) allows the cancellation across episodes, we can guarantee $O(\sqrt{K}\log K)$ violation when the cancellation is not allowed, as defined in (5) in \cite{EfrManPir_20}, by making some minor modification to the algorithm. In particular, in the policy refinement and identification phases, instead of using the $M$ greedy policy in a round-robin fashion, we can use a mixed policy that chooses policy $m$ with probability $\alpha_m.$ With that modification, we will   have
\begin{itemize}
    \item The Pruning phase consists at most $20H^2SAK^{0.25}\log K$ episodes, so resulting in at most $\tilde{\cal O}(H^2SAK^{0.25})$ violation. 

    \item During the refinement phase, we have shown that with a high probability, the mixed policy used in the $t$th iteration has a reward gap of ${\cal O}\left(H\frac{\log K}{\sqrt{t\sqrt{K}}}\right)$ for $t\geq 2.$ The same result holds for constraint violation. The $t$th iteration includes $\sqrt{K}$ episodes and the refinement phase includes $\sqrt{K}$ iterations. Therefore, with a high probability, the total violation without canceling across episodes is ${\cal O}(H\sqrt{K}\log K).$

    \item The mixed policy used in the policy identification phase has a constraint violation bound of ${\cal O}\left(\frac{\log K}{\sqrt{K}}\right)-\epsilon_\rho<0$ for a sufficiently large $K.$ So the violation is zero with a high probability. 
\end{itemize}
Therefore, the total violation is bounded by $\tilde{\cal O}(H\sqrt{K})$ by using definition (5) in \cite{EfrManPir_20}.

\section{$\Omega(H\sqrt{K})$ lower bound with $\sigma_{\min}>0$}
\label{app:lb}
While in spirit, $\sigma_{\min}$ is similar to the reward gap in bandits, we need to point out that for safe linear bandits, it has been shown in the literature \cite{CheGanSal_22} that logarithmic bounds are impossible even when $\sigma_{\min}>0$ and independent of $K.$ We next present an $H\sqrt{K}$ lower bound for CMDPs, using an example inspired by the impossibility example in Section 4 of \cite{CheGanSal_22}.

Consider a CMDP with four states $\{0, 1, 2, 3\}$, two actions $\{u, v\},$ and $x_1=0.$ When action $u$ is taken at step 1, the MDP moves to state $1$ and the agent receives zero reward and zero utility. After step 1, if the MDP is in state 1, it remains in state 1 regardless of the action and the agent receives reward 1 and zero utility at each step. 

When action $v$ is taken at step 1, the MDP moves to state 2 with probability $p$ and state 3 with probability $1-p.$ The agent receives zero reward and zero utility. After step 1, if the agent is in state $2,$  it remains in state 2 regardless of the action and the agent receives zero reward and utility 1 at each step.  If the agent is in state $3,$  it remains in state 3 independent of the action and the agent receives zero reward and utility $0.5$ at each step. 
The CMDP is 
\begin{align*}
&\max_{\{\pi_k\}} \mathbb{E}\left[\sum_{k=1}^K\sum_{i=1}^H r_{k,i}(x_{k,i}, a_{k,i})\right]\\
\hbox{subject to: } &  \mathbb{E}\left[\sum_{k=1}^K\sum_{i=1}^H g_{k,i}(x_{k,i}, a_{k,i})\right] \geq  \frac{H-1}{3}.
\end{align*}
Since the reward and utility received after step 2 are action-independent and the initial state is fixed, the problem can be simplified to decide the action at the first step:  
\begin{align*}
&\max_{\pi_v} (H-1)(1-\pi_v)\\
\hbox{subject to: } & \pi_v(H-1)\left(\frac{1}{2}+\frac{p}{2}\right)\geq \frac{H-1}{3}
\end{align*} where $\pi_v$ is the probability of taking action $v$ at step 1. Note that $p$ is unknown and needs to be learned from samples.  

It is straightforward to see that the optimal solution to the problem above is $\pi^*_v=\frac{2}{3(1+p)}$ and $\sigma_{\min}=(H-1)\min\left\{1-\frac{2}{3(1+p)},\frac{1}{2}\right\}\geq \frac{H-1}{3}$ where the first term is the reward gap by restricting the action space to $\{v\}$ and the second term is the constraint violation when restricting the action space to $\{u\}.$ It is also easy to see that the associated LP satisfies Slater's condition (e.g., choosing $\pi_v=0.9.$)

We now consider two instances $p=0.5$ and $p=0.5+\kappa.$ 
We now consider an online algorithm without knowing which instance it is. We assume the online algorithm chooses policy $\pi_{k,v}\in\left[\frac{2}{3(1.5+\kappa)}, \frac{2}{4.5}\right]$ at each step because the optimal policies for the two instances belong to this interval. 

Given an online learning algorithm, 
the regret is 
\begin{align*}
&(H-1)K\left(1-\frac{2}{3(1+p)}\right)-\mathbb{E}_p\left[\sum_{k=1}^K\sum_{i=1}^H r_{k,i}(x_{k,i}, a_{k,i})\right]\\
=&(H-1)K\left(1-\frac{2}{3(1+p)}\right)-(H-1)\mathbb{E}_p\left[\sum_{k=1}^K 1-\mathbb{I}(a_{k,1}=v)\right]\\
=&(H-1)\left(\sum_{i=1}^K \mathbb{E}_p\left[\pi_{k,v}\right]-\frac{2K}{3(1+p)}\right) ,   
\end{align*} where $\mathbb{E}_p$ denotes the expectation for the instance with parameter $p,$
 and the constraint violation is 
 \begin{align*}
&\mathbb{E}_p\left[\sum_{k=1}^K\left(\frac{H-1}{3}-\sum_{i=1}^H g_{k,i}(x_{k,i}, a_{k,i})\right)\right]\\
=&\frac{(H-1)K}{3}-\sum_{k=1}^K \mathbb{E}_p\left[\mathbb{E}_p\left[\left.\sum_{i=1}^H g_{k,i}(x_{k,i}, a_{k,i})\right|a_{k,1}\right]\right]\\
=&\frac{(H-1)K}{3}-\sum_{k=1}^K \mathbb{E}_p\left[\mathbb{I}(a_{k,1}=v)(H-1)(\frac{1}{2}+\frac{p}{2})\right]\\
=&{(H-1)}\left(\frac{K}{3}-\frac{1+p}{2} \sum_{k=1}^K \mathbb{E}_p\left[\pi_{k,v}\right] \right).
\end{align*}

Now choose $\kappa=1/\sqrt{K}$ and define event $\Phi=\left\{\sum_{k=1}^K \mathbb{I}( \pi_{k,v}\geq \frac{2}{3(1.5+\frac{1}{2\sqrt{K}})})\geq \frac{K}{2}\right\}.$ Considering instance 2, we have
\begin{align*}
\hbox{regret}_2 \geq &  (H-1) \frac{K}{2}\left(\frac{2}{3(1.5+\frac{1}{2\sqrt{K}})}- \frac{2}{3(1.5+\frac{1}{\sqrt{K}})}\right)P_{0.5+1/\sqrt{K}}(\Phi)\\
=& \frac{(H-1)\sqrt{K}}{6(1.5+\frac{1}{2\sqrt{K}})(1.5+\frac{1}{\sqrt{K}})}P_{0.5+1/\sqrt{K}}(\Phi).
\end{align*}
Considering instance 1, we have $$ \hbox{violation}_1 \geq (H-1)\frac{K}{2}\left(\frac{1}{3}- \frac{3}{4} \frac{2}{3(1.5+\frac{1}{2\sqrt{K}})}\right)P_{0.5}(\Phi^c)=\frac{(H-1)\sqrt{K}}{6(3+\frac{1}{\sqrt{K}})}P_{0.5}(\Phi^c).$$
Therefore, we have for sufficiently large $K,$
\begin{align*}
\hbox{violation}_1+\hbox{regret}_2 \geq    \frac{(H-1)\sqrt{K}}{20}\left(P_{0.5+1/\sqrt{K}}(\Phi)+P_{0.5}(\Phi^c)\right).
\end{align*} Invoking the Bretagnolle-Huber inequality (Theorem 14.2 in \cite{LatSze_20}), we have 
\begin{align*}
\hbox{violation}_1+\hbox{regret}_2 \geq&    \frac{(H-1)\sqrt{K}}{40}\exp\left(D(P_{0.5}\|P_{0.5+1/\sqrt{K}})\right)\\
=&  \frac{(H-1)\sqrt{K}}{40}\exp\left(K\left(0.5\log \frac{0.5}{0.5+1/\sqrt{K}}+0.5\log\frac{0.5}{0.5-1/\sqrt{K}}\right)\right)\\
=&  \frac{(H-1)\sqrt{K}}{40}\exp\left(-\frac{K}{2}\log (1-4/{K})\right)\\
\geq & \frac{e^2}{40}(H-1)\sqrt{K},
\end{align*} where the last inequality holds when $K>4$ and is based on fact that $\log(1+x)\leq x$ for $x>-1.$

Therefore, under any online learning algorithm, at least one of the instances has $\Omega(H\sqrt{K})$ regret or violation even when both instances have $\sigma_{\min}=\frac{H-1}{3}.$ As pointed out in  \cite{CheGanSal_22}, this impossibility result is because of the precision required in the optimal solution, i.e., the optimal solution needs to find the exact randomization, which is very sensitive to estimation error, which is fundamentally different from unconstrained problems where we can show logarithmic regret if the instance is well separated, i.e., $\sigma_{\min}>0$ and independent of $K.$ This lower bound matches the upper bound up to a polylogarithmic factor.

\section{Simulations}

\subsection{Synthetic CMDP} \label{app:sim-mdp}
In the systehtic CMDP, we choose $\vert \cS\vert = 3,\vert \cA\vert=3,H=3.$ The detailed parameters of the CMDP in the first experiment are shown in Table \ref{tab:tran}, \ref{tab:re_r} and \ref{tab:re_u}.

We executed the pruning phase (Triple-Q) over $100,000$ episodes and repeated it for $10$ times, followed by the refinement phase and identification phase over $7,000,000$ episodes for each. Since the problem has only one constraint, the optimal policy has only one stochastic decision, which can be decided by evaluating the frequencies of two greedy policies. The identified policy is shown in Table \ref{tab:ide}. We can see that the policy only takes one stochastic decision. The cumulative reward and cumulative utility we get for our learned policy are $1.561$ and $2.016,$ which is close to the optimal solution $1.573$ and $2.$

\begin{table*}[!ht]
    \centering
    \caption{Transition Kernels (the rows represent (previous state, action) and the columns represent (step, next state)).}
    \resizebox{\textwidth}{14mm}{
    \begin{tabular}{|c|c|c|c|c|c|c|c|c|c|}
        \hline
        & (1,1) & (1,2) & (1,3) & (2,1) & (2,2) & (2,3) & (3,1) & (3,2) & (3,3)\\
        \hline
        (1,1) & 0.3112981 & 0.35107633 & 0.27041442 & 0.42626645 & 0.04822746 & 0.14663183 & 0.4031534 & 0.19783729 & 0.39831431\\
        \hline
        (1,2) & 0.23314339 & 0.32491141 & 0.48360071 & 0.24246185 & 0.19021328 & 0.43972054 & 0.26457139 & 0.21435897 & 0.26256243 \\
        \hline
        (1,3) & 0.45555851 & 0.32401226 & 0.24598487 & 0.3312717 & 0.76155926 & 0.41364763 & 0.33227521 & 0.58780374 & 0.33912326\\
        \hline
        (2,1) & 0.32676574 & 0.35320112 & 0.1300059 & 0.35453348 & 0.32114495 & 0.40817113 & 0.1762648 & 0.30097191 & 0.48437535\\
        \hline
        (2,2) & 0.11092341 & 0.28034838 & 0.45655888 & 0.23441632 & 0.2847394 & 0.235718 & 0.17239783 & 0.37273618 & 0.08000908\\
        \hline
        (2,3) & 0.56231085 & 0.3664505 & 0.41343525 & 0.4110502 & 0.39411565 & 0.35611087 & 0.65133738 & 0.32629191 & 0.43561556\\
        \hline
    \end{tabular}} \label{tab:tran}
\end{table*}

\begin{table*}[!ht]
    \centering
    \caption{Rewards (the rows represent (state, action) and the columns represent step.)}
    \resizebox{\textwidth}{8mm}{
    \begin{tabular}{|c|c|c|c|c|c|c|c|c|c|}
        \hline
        & (1,1) & (1,2) & (1,3) & (2,1) & (2,2) & (2,3) & (3,1) & (3,2) & (3,3)\\
        \hline
        1 & 0.5507979 & 0.70814782 & 0.29090474 & 0.51082761 & 0.89294695 & 0.89629309 & 0.12558531 & 0.20724388 & 0.0514672\\
        \hline
        2 & 0.44080984 & 0.02987621 & 0.45683322 & 0.64914405 & 0.27848728 & 0.6762549 & 0.59086282 & 0.02398188 & 0.55885409\\
        \hline
        3 & 0.25925245 & 0.4151012 & 0.28352508 & 0.69313792 & 0.44045372 & 0.15686774 & 0.54464902 & 0.78031476 & 0.30636353\\
        \hline
    \end{tabular}} \label{tab:re_r}
\end{table*}

\begin{table*}[!ht]
    \centering
    \caption{Utilities (the rows represent (state, action) and the columns represent step. )}
    \resizebox{\textwidth}{8mm}{
    \begin{tabular}{|c|c|c|c|c|c|c|c|c|c|}
        \hline
       & (1,1) & (1,2) & (1,3) & (2,1) & (2,2) & (2,3) & (3,1) & (3,2) & (3,3)\\
        \hline
        1 & 0.22195788 & 0.38797126 & 0.93638365 & 0.97599542 & 0.67238368 & 0.90283411 & 0.84575087 & 0.37799404 & 0.09221701\\
        \hline
        2 & 0.6534109 & 0.55784076 & 0.36156476 & 0.2250545 & 0.40651992 & 0.46894025 & 0.26923558 & 0.29179277 & 0.4576864\\
        \hline
        3 & 0.86053391 & 0.5862529 & 0.28348786 & 0.27797751 & 0.45462208 & 0.20541034 & 0.20137871 & 0.51403506 & 0.08722937\\
        \hline
    \end{tabular}}\label{tab:re_u}
\end{table*}

\begin{table*}[!ht]
    \centering
    \caption{Identified Policy $\tilde{\pi}_h(a|x)$ (the rows represent (state, action) and the columns represent step.)}
    \begin{tabular}{|c|c|c|c|c|c|c|c|c|c|}
        \hline
       & (1,1) & (1,2) & (1,3) & (2,1) & (2,2) & (2,3) & (3,1) & (3,2) & (3,3)\\
        \hline
        1 & 0 & 0.059305 & 0.940695 & 0 & 0 & 1 & 1 & 0 & 0\\
        \hline
        2 & 1 & 0 & 0 & 0 & 0 & 1 & 0 & 0 & 1\\
        \hline
        3 & 1 & 0 & 0 & 1 & 0 & 0 & 0 & 1 & 0\\
        \hline
    \end{tabular}\label{tab:ide}
\end{table*}

\subsection{Grid World} \label{sec:app-grid}

\begin{figure}[!ht]
    \centering
    \includegraphics[scale = 0.3]{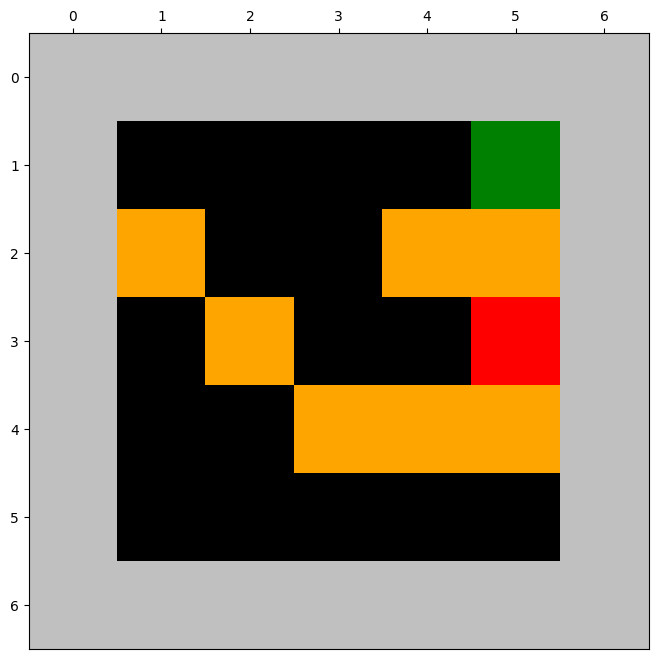}
    \caption{Grid World}
    \label{fig:env}
\end{figure}

As shown in Figure \ref{fig:env}, the task of the agent is to go from the red grid point to the green grid point.  The black grid points are the {\em safe} points over which the agent can move, and the yellow grid points are obstacles. Moving over an obstacle incurs a penalty of one. The constraint is that the agent can incur only an average cost of $0.5$ or less.  The agent can take six steps at maximum. The reward associated with reaching the destination is $1,$ and the rewards for other locations, after six steps, are the Euclidean distance from the location to the destination (normalized by the longest distance). At each grid point, the agent has five actions to choose from: up, down, left, right, and stay, except at the boundary. The goal is to maximize the reward subject to the constraint.

During the experiment, we observed that policy pruning is much more efficient than the theoretical worst case. For this specific environment, the optimal policy should have $6\times 5\times 5 + 1= 155$ nonzero $\pi^*_h(a|x)$'s. After the first phase (Triple-Q), we have roughly $200$ (step, state, action) triples (here "roughly" considers the difference among different trials with different random seeds), associated with stochastic decisions to check and prune. Except for the two ``necessary'' decisions, which are stochastic decisions in the optimal policy, for all trials, the algorithm only checked two candidate triples and eliminated the rest candidate triples in the process. The identified policy is shown in Figure \ref{fig:policy_1} to \ref{fig:policy_6}. The arrow means what direction the agent will go and circle means the agent prefers to stay in the current state.

\begin{figure}[!ht]
\centering
   \begin{subfigure}[b]{0.33\textwidth}
     \centering
    \includegraphics[scale = 0.25]{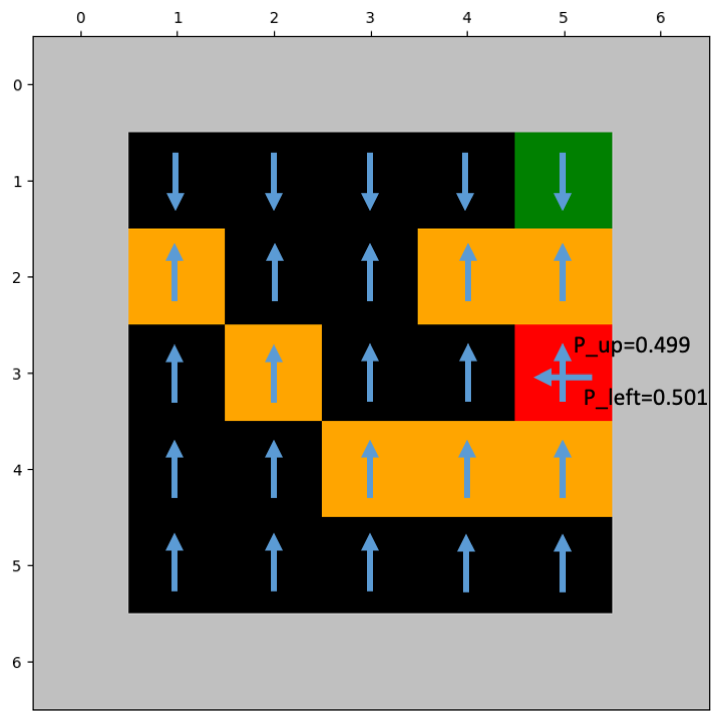}
    \caption{identified policy when h = 1}
    \label{fig:policy_1}
     \end{subfigure}
    \begin{subfigure}[b]{0.33\textwidth}
  \centering
    \includegraphics[scale = 0.25]{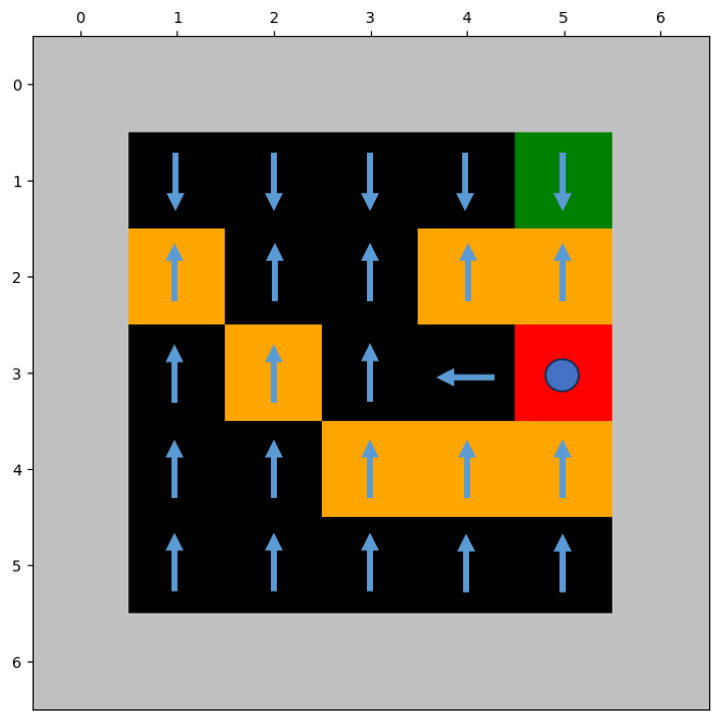}
    \caption{identified policy when h = 2}
     \end{subfigure}
     \begin{subfigure}[b]{0.33\textwidth}
            \centering
    \includegraphics[scale = 0.25]{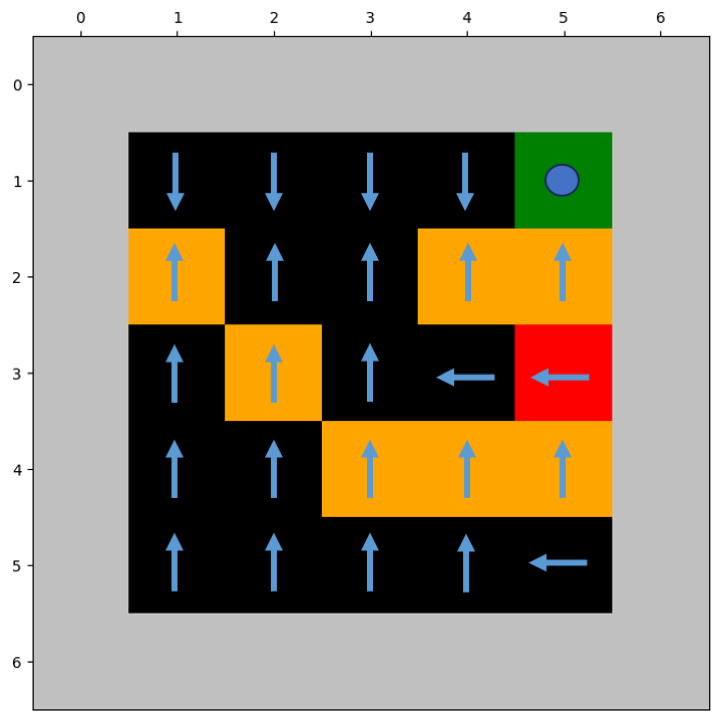}
    \caption{identified policy when h = 3}
     \end{subfigure}
     \begin{subfigure}[b]{0.33\textwidth}
     \centering
    \includegraphics[scale = 0.25]{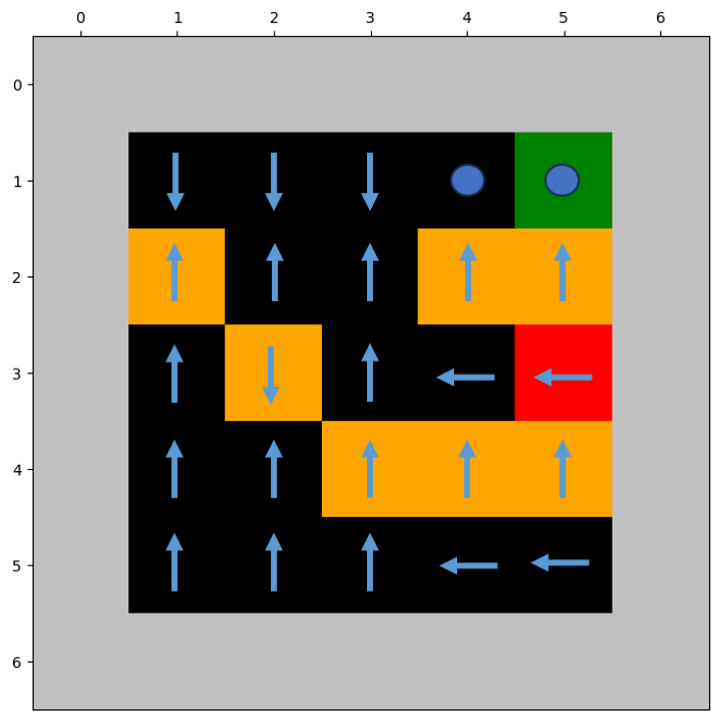}
    \caption{identified policy when h = 4}
    \end{subfigure}
     \begin{subfigure}[b]{0.33\textwidth}
      \centering
    \includegraphics[scale = 0.25]{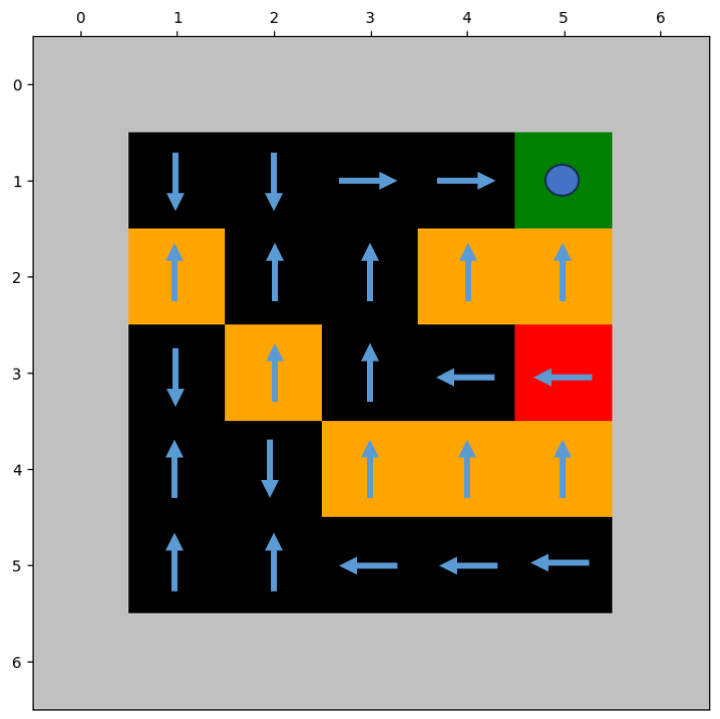}
    \caption{identified policy when h = 5}
    \end{subfigure}
     \begin{subfigure}[b]{0.33\textwidth}
      \centering
    \includegraphics[scale = 0.25]{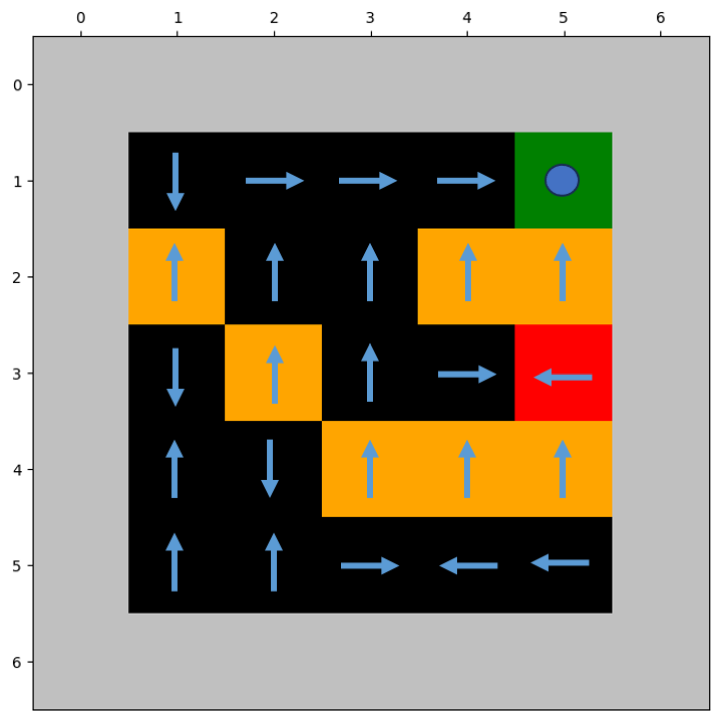}
    \caption{identified policy when h = 6}
    \label{fig:policy_6}
    \end{subfigure}
\caption{The policy identified by PRI in the Grid World}
\end{figure}
\end{document}

%% file: custom.tex
\newcommand{\cA}{\mathcal A}

\newcommand{\cS}{\mathcal S}

\hypersetup{
	colorlinks=true,
	linkcolor=black,
	filecolor=blue,
	citecolor = blue,      
	urlcolor=blue,
}

%% file: main_arxiv.bbl
\begin{thebibliography}{}

\bibitem[Abeyruwan et~al., 2023]{abeyruwan2023sim2real}
Abeyruwan, S.~W., Graesser, L., D’Ambrosio, D.~B., Singh, A., Shankar, A.,
  Bewley, A., Jain, D., Choromanski, K.~M., and Sanketi, P.~R. (2023).
\newblock i-sim2real: Reinforcement learning of robotic policies in tight
  human-robot interaction loops.
\newblock In {\em Conference on Robot Learning}, pages 212--224. PMLR.

\bibitem[Agarwal et~al., 2020]{AgaKakKri_20}
Agarwal, A., Kakade, S., Krishnamurthy, A., and Sun, W. (2020).
\newblock Flambe: Structural complexity and representation learning of low rank
  mdps.
\newblock In {\em Advances Neural Information Processing Systems (NeurIPS)},
  page 20095–20107. Curran Associates Inc.

\bibitem[Agarwal et~al., 2021]{AgaBaiAgg_21}
Agarwal, M., Bai, Q., and Aggarwal, V. (2021).
\newblock Markov decision processes with long-term average constraints.
\newblock {\em arXiv preprint arXiv:2106.06680}.

\bibitem[Agarwal et~al., 2022]{AgaBaiAgg_22}
Agarwal, M., Bai, Q., and Aggarwal, V. (2022).
\newblock Regret guarantees for model-based reinforcement learning with
  long-term average constraints.
\newblock In {\em Uncertainty in Artificial Intelligence}, pages 22--31. PMLR.

\bibitem[Altman, 1999]{Alt_99}
Altman, E. (1999).
\newblock {\em Constrained Markov decision processes}, volume~7.
\newblock CRC Press.

\bibitem[Azar et~al., 2013]{AzaRemHil_13}
Azar, M.~G., Munos, R., and Kappen, H.~J. (2013).
\newblock Minimax pac bounds on the sample complexity of reinforcement learning
  with a generative model.
\newblock {\em Mach. Learn.}, 91(3):325–349.

\bibitem[Brantley et~al., 2020]{BraDudLyk_20}
Brantley, K., Dudik, M., Lykouris, T., Miryoosefi, S., Simchowitz, M.,
  Slivkins, A., and Sun, W. (2020).
\newblock Constrained episodic reinforcement learning in concave-convex and
  knapsack settings.
\newblock In {\em Advances Neural Information Processing Systems (NeurIPS)},
  volume~33, pages 16315--16326. Curran Associates, Inc.

\bibitem[Bura et~al., 2021]{BurHasKal_21}
Bura, A., HasanzadeZonuzy, A., Kalathil, D., Shakkottai, S., and Chamberland,
  J.-F. (2021).
\newblock Safe exploration for constrained reinforcement learning with provable
  guarantees.
\newblock {\em arXiv preprint arXiv:2112.00885}.

\bibitem[Chen et~al., 2022a]{CheJaiLuo_22}
Chen, L., Jain, R., and Luo, H. (2022a).
\newblock Learning infinite-horizon average-reward markov decision process with
  constraints.
\newblock In {\em Int. Conf. Machine Learning (ICML)}, pages 3246--3270. PMLR.

\bibitem[Chen et~al., 2022b]{CheGanSal_22}
Chen, T., Gangrade, A., and Saligrama, V. (2022b).
\newblock Doubly-optimistic play for safe linear bandits.
\newblock {\em arXiv preprint arXiv:2209.13694}.

\bibitem[Ding et~al., 2021]{DinWeiYan_20}
Ding, D., Wei, X., Yang, Z., Wang, Z., and Jovanovic, M. (2021).
\newblock Provably efficient safe exploration via primal-dual policy
  optimization.
\newblock In {\em Int. Conf. Artificial Intelligence and Statistics (AISTATS)},
  volume 130, pages 3304--3312. PMLR.

\bibitem[Domingues et~al., 2021]{DomMenKau_21}
Domingues, O.~D., M{\'e}nard, P., Kaufmann, E., and Valko, M. (2021).
\newblock Episodic reinforcement learning in finite mdps: Minimax lower bounds
  revisited.
\newblock In {\em Algorithmic Learning Theory}, pages 578--598. PMLR.

\bibitem[Efroni et~al., 2020]{EfrManPir_20}
Efroni, Y., Mannor, S., and Pirotta, M. (2020).
\newblock Exploration-exploitation in constrained {MDP}s.
\newblock {\em arXiv preprint arXiv:2003.02189}.

\bibitem[Even-Dar et~al., 2006]{EveEyaShi_06}
Even-Dar, E., Mannor, S., Mansour, Y., and Mahadevan, S. (2006).
\newblock Action elimination and stopping conditions for the multi-armed bandit
  and reinforcement learning problems.
\newblock {\em Journal of machine learning research}, 7(6).

\bibitem[Ghosh et~al., 2022]{GhoZhoShr_22}
Ghosh, A., Zhou, X., and Shroff, N. (2022).
\newblock Provably efficient model-free constrained rl with linear function
  approximation.
\newblock In {\em NeurIPS}.

\bibitem[He et~al., 2021]{HeZhoGu_21}
He, J., Zhou, D., and Gu, Q. (2021).
\newblock Nearly minimax optimal reinforcement learning for discounted mdps.
\newblock {\em Advances in Neural Information Processing Systems},
  34:22288--22300.

\bibitem[Jin et~al., 2018]{JinAllBub_18}
Jin, C., Allen-Zhu, Z., Bubeck, S., and Jordan, M.~I. (2018).
\newblock Is q-learning provably efficient?
\newblock In {\em Advances Neural Information Processing Systems (NeurIPS)},
  volume~31, pages 4863--4873.

\bibitem[Koole, 1988]{Koo_88}
Koole, G. (1988).
\newblock Stochastische dynamische programmering met bijvoorwaarden
  (translation: Stochastic dynamic programming with additional constraints).
\newblock {\em Master's thesis, Leiden University}.

\bibitem[Lattimore and Szepesvári, 2020]{LatSze_20}
Lattimore, T. and Szepesvári, C. (2020).
\newblock {\em Bandit Algorithms}.
\newblock Cambridge University Press.

\bibitem[Lindegaard et~al., 2023]{lindegaard2023intrinsic}
Lindegaard, M., Vinje, H.~J., and Severinsen, O.~A. (2023).
\newblock Intrinsic rewards from self-organizing feature maps for exploration
  in reinforcement learning.
\newblock {\em arXiv preprint arXiv:2302.04125}.

\bibitem[Liu et~al., 2023]{liu2023exploring}
Liu, P., Zhou, J., and Lv, J. (2023).
\newblock Exploring the first-move balance point of go-moku based on
  reinforcement learning and monte carlo tree search.
\newblock {\em Knowledge-Based Systems}, 261:110207.

\bibitem[Liu et~al., 2021]{LiuZhoKal_21}
Liu, T., Zhou, R., Kalathil, D., Kumar, P., and Tian, C. (2021).
\newblock Learning policies with zero or bounded constraint violation for
  constrained {MDPs}.
\newblock In {\em Advances Neural Information Processing Systems (NeurIPS)},
  volume~34.

\bibitem[Moskovitz et~al., 2023]{MosOdoVee_23}
Moskovitz, T., O'Donoghue, B., Veeriah, V., Flennerhag, S., Singh, S., and
  Zahavy, T. (2023).
\newblock Reload: Reinforcement learning with optimistic ascent-descent for
  last-iterate convergence in constrained mdps.
\newblock {\em arXiv preprint arXiv:2302.01275}.

\bibitem[Rajawat et~al., 2023]{rajawat2023cognitive}
Rajawat, A.~S., Goyal, S., Chauhan, C., Bedi, P., Prasad, M., and Jan, T.
  (2023).
\newblock Cognitive adaptive systems for industrial internet of things using
  reinforcement algorithm.
\newblock {\em Electronics}, 12(1):217.

\bibitem[Ross, 1989]{Ros_89}
Ross, K.~W. (1989).
\newblock Randomized and past-dependent policies for markov decision processes
  with multiple constraints.
\newblock {\em Operations Research}, 37(3):474--477.

\bibitem[Sidford et~al., 2018]{SidWanWu_18}
Sidford, A., Wang, M., Wu, X., Yang, L., and Ye, Y. (2018).
\newblock Near-optimal time and sample complexities for solving markov decision
  processes with a generative model.
\newblock {\em Advances in Neural Information Processing Systems}, 31.

\bibitem[Singh et~al., 2020]{SinGupShr_20}
Singh, R., Gupta, A., and Shroff, N.~B. (2020).
\newblock Learning in markov decision processes under constraints.
\newblock {\em arXiv preprint arXiv:2002.12435}.

\bibitem[Taupin et~al., 2022]{TauJedPro_22}
Taupin, J., Jedra, Y., and Proutiere, A. (2022).
\newblock Best policy identification in linear mdps.
\newblock {\em arXiv preprint arXiv:2208.05633}.

\bibitem[Wei et~al., 2023]{WeiGhoShr_23}
Wei, H., Ghosh, A., Shroff, N., Ying, L., and Zhou, X. (2023).
\newblock Provably efficient model-free algorithms for non-stationary cmdps.
\newblock In {\em Int. Conf. Artificial Intelligence and Statistics (AISTATS)},
  volume 206, pages 6527--6570.

\bibitem[Wei et~al., 2021]{WeiLiuYin_21}
Wei, H., Liu, X., and Ying, L. (2021).
\newblock A provably-efficient model-free algorithm for constrained markov
  decision processes.
\newblock {\em arXiv preprint arXiv:2106.01577}.

\bibitem[Wei et~al., 2022a]{WeiLiuYin_22-2}
Wei, H., Liu, X., and Ying, L. (2022a).
\newblock A provably-efficient model-free algorithm for infinite-horizon
  average-reward constrained markov decision processes.
\newblock In {\em AAAI Conf. Artificial Intelligence}.

\bibitem[Wei et~al., 2022b]{WeiLiuYin_22}
Wei, H., Liu, X., and Ying, L. (2022b).
\newblock {Triple-Q:} a model-free algorithm for constrained reinforcement
  learning with sublinear regret and zero constraint violation.
\newblock In {\em Int. Conf. Artificial Intelligence and Statistics (AISTATS)}.

\bibitem[Zheng and Ratliff, 2020]{ZheRat_20}
Zheng, L. and Ratliff, L. (2020).
\newblock Constrained upper confidence reinforcement learning.
\newblock In {\em Learning for Dynamics and Control}, pages 620--629. PMLR.

\end{thebibliography}
